\documentclass[11pt]{article}

\usepackage[utf8x]{inputenc}
\usepackage[T1]{fontenc}
\usepackage{booktabs} 

\usepackage[vlined,ruled]{algorithm2e} 

\SetAlFnt{\small}
\SetAlCapFnt{\small}
\SetAlCapNameFnt{\small}
\SetAlCapHSkip{0pt}
\IncMargin{-\parindent}

\usepackage{xcolor}         
\usepackage{hyperref}

\newcommand{\declarecolor}[2]{\definecolor{#1}{RGB}{#2}\expandafter\newcommand\csname #1\endcsname[1]{\textcolor{#1}{##1}}}
\declarecolor{White}{255, 255, 255}
\declarecolor{Black}{0, 0, 0}
\declarecolor{Maroon}{128, 0, 0}
\declarecolor{Coral}{255, 127, 80}
\declarecolor{Red}{182, 21, 21}
\declarecolor{LimeGreen}{50, 205, 50}
\declarecolor{DarkGreen}{0, 80, 0}
\declarecolor{Purple}{146, 42, 158}
\declarecolor{Navy}{0, 0, 128}
\declarecolor{LightBlue}{84, 101, 202}
\definecolor{mydarkblue}{rgb}{0,0.08,0.45}
\hypersetup{ %
    pdftitle={},
    pdfkeywords={},
    pdfborder=0 0 0,
    pdfpagemode=UseNone,
    colorlinks=true,
    linkcolor=Navy,
    citecolor=DarkGreen,
    filecolor=Purple,
    urlcolor=Purple,
}

\usepackage{fullpage}

\usepackage{amsthm}
\usepackage{amsmath}
\usepackage{amsfonts}
\usepackage{thm-restate}

\usepackage{tikz}
\usetikzlibrary{shapes.geometric, arrows, positioning}

\usepackage{multirow}

\usepackage{nicefrac}

\usepackage{enumitem} 

\usepackage[all]{nowidow} 

\usepackage{lipsum} 

\usepackage[numbers]{natbib}
\usepackage{physics}
\usepackage{bm}
\usepackage{bbm}
\usepackage{xspace}
\usepackage{complexity}
\usepackage{mleftright}
\usepackage{xcolor}
\mleftright

\def\vg{{\bm{g}}}

\def\vm{{\bm{m}}}

\def\vu{{\bm{u}}}
\def\vv{{\bm{v}}}

\def\vx{{\bm{x}}}
\def\vy{{\bm{y}}}

\renewcommand\vec\bm

\def\mA{{\mathbf{A}}}

\def\mG{{\mathbf{G}}}

\def\mI{{\mathbf{I}}}

\def\mK{{\mathbf{K}}}

\def\mU{{\mathbf{U}}}

\def\cA{{\mathcal{A}}}
\def\cB{{\mathcal{B}}}
\def\cC{{\mathcal{C}}}
\def\cD{{\mathcal{D}}}
\def\cE{{\mathcal{E}}}
\def\cF{{\mathcal{F}}}

\def\cM{{\mathcal{M}}}

\def\cO{{\mathcal{O}}}

\def\cQ{{\mathcal{Q}}}

\def\cU{{\mathcal{U}}}
\def\cV{{\mathcal{V}}}

\def\cX{{\mathcal{X}}}
\def\cY{{\mathcal{Y}}}
\def\cZ{{\mathcal{Z}}}

\newcommand{\phireg}{\Phi\text{-}\op{Reg}}
\newcommand{\reg}{\op{Reg}}

\newcommand{\delimit}[3]{\newcommand{#1}[1]{\left#2##1\right#3}}
\delimit \ceil \lceil \rceil
\delimit \floor \lfloor \rfloor

\DeclareMathOperator*{\argmin}{argmin}
\DeclareMathOperator*{\argmax}{argmax}
\let\E\relax
\DeclareMathOperator*{\E}{\mathbb E}

\DeclareMathOperator*{\vol}{vol}

\newcommand{\cons}{\bm{c}}

\renewcommand{\R}{\mathbb R}
\newcommand{\N}{\mathbb N}

\let\op\operatorname
\let\eps\epsilon

\let\ip\ev

\newcommand{\enfuns}{\Phi^m}

\newcommand{\eah}{\textsf{EAH}}
\newcommand{\sep}{\textsf{SEP}}
\newcommand{\ger}{\textsf{GER}}
\newcommand{\either}{\textsf{GERorSEP}}
\newcommand{\shellelips}{\textsf{ShellEllipsoid}}
\newcommand{\shellgd}{\textsf{ShellGD}}
\newcommand{\shellproj}{\textsf{ShellProject}}

\newcommand{\fp}{\textsf{FP}}

\newcommand{\regbox}{\mathfrak{R}}

\newcommand{\tilY}{\tilde{\cY}}
\newcommand{\tilPhi}{\tilde{\Phi}}
\newcommand{\tilphi}{\tilde{\phi}}

\newcommand{\supp}{\op{supp}}

\newcommand{\defeq}{:=}
\newcommand{\range}[1]{[#1]}

\renewcommand{\co}{\op{conv}}
\DeclareMathOperator{\sign}{sign}

\newcommand{\pr}{\mathbb{P}}
\newcommand{\V}{\mathbb{V}}

\newcommand{\lineP}{\overline{P}}
\newcommand{\linem}{\overline{m}}
\newcommand{\linex}{\overline{\vx}}

\usepackage[capitalize,noabbrev,nameinlink]{cleveref}

\usepackage{MnSymbol}

\renewcommand\vec\bm

\theoremstyle{plain}
\newtheorem{theorem}{Theorem}[section]
\newtheorem{proposition}[theorem]{Proposition}
\newtheorem{lemma}[theorem]{Lemma}
\newtheorem{corollary}[theorem]{Corollary}
\theoremstyle{definition}
\newtheorem{definition}[theorem]{Definition}
\newtheorem{assumption}[theorem]{Assumption}

\newtheorem{remark}[theorem]{Remark}

\usepackage{autonum}

\usepackage{authblk}

\title{Learning and Computation of $\Phi$-Equilibria \\at the Frontier of Tractability}

\author[1]{Brian Hu Zhang\thanks{Equal contribution.}}
\author[1]{Ioannis Anagnostides$^*$}
\author[1,2]{Emanuel Tewolde}
\author[1,2]{Ratip Emin Berker}
\author[3]{Gabriele Farina}
\author[1,2,4]{Vincent Conitzer}
\author[1,5]{Tuomas Sandholm}

\affil[1]{Carnegie Mellon University}
\affil[2]{Foundations of Cooperative AI Lab (FOCAL)}
\affil[3]{Massachusetts Institute of Technology}
\affil[4]{University of Oxford}
\affil[5]{Additional affiliations: Strategy Robot, Inc., Strategic Machine, Inc., Optimized Markets, Inc.}
\affil[ ]{}
\affil[ ]{\texttt{\{bhzhang,ianagnos,etewolde,rberker,conitzer,sandholm\}}\texttt{@cs.cmu.edu}, \texttt{gfarina}\texttt{@mit.edu}}

\begin{document} 

\maketitle

\pagenumbering{gobble}

\begin{abstract}
\emph{$\Phi$-equilibria}---and the associated notion of \emph{$\Phi$-regret}---are a powerful and flexible framework at the heart of online learning and game theory, whereby enriching the set of deviations $\Phi$ begets stronger notions of rationality. Recently, Daskalakis, Farina, Fishelson, Pipis, and Schneider (STOC '24)---abbreviated as DFFPS---settled the existence of efficient algorithms when $\Phi$ contains only linear maps under a general, $d$-dimensional convex constraint set $\mathcal{X}$. In this paper, we significantly extend their work by resolving the case where $\Phi$ is $k$-dimensional; degree-$\ell$ polynomials constitute a canonical such example with $k = d^{O(\ell)}$. In particular, positing only oracle access to $\cX$, we obtain two main positive results:
\begin{itemize}
    \item a $\text{poly}(n, d, k, \text{log}(1/\epsilon))$-time algorithm for computing $\epsilon$-approximate $\Phi$-equilibria in $n$-player multilinear games, and
    \item an efficient online algorithm that incurs  average $\Phi$-regret at most $\epsilon$ using $\text{poly}(d, k)/\epsilon^2$ rounds.
\end{itemize}
We also show nearly matching---up to constant factors in the exponents---lower bounds parameterized by $k$ in the online learning setting, thereby obtaining for the first time a family of deviations that captures the 
learnability of $\Phi$-regret.

From a technical standpoint, we extend the framework of DFFPS from linear maps to the more challenging case of maps with polynomial dimension. At the heart of our approach is a polynomial-time algorithm for computing an \emph{expected fixed point} of any $\phi: \cX \to \cX$---that is, a distribution $\mu \in \Delta(\cX)$ such that $\E_{\vec{x} \sim \mu}[ \phi(\vx) - \vx ] \approx 0$---based on the seminal \emph{ellipsoid against hope (EAH)} algorithm of Papadimitriou and Roughgarden (JACM '08). In particular, our algorithm for computing $\Phi$-equilibria is based on executing EAH in a nested fashion---each step of EAH itself being implemented by invoking a separate call to EAH.
\end{abstract}

\clearpage

\tableofcontents

\clearpage

\pagenumbering{arabic}

\section{Introduction}

What constitutes a solution to a game? Descriptive and prescriptive theories in economics brought together in response to such questions chiefly revolve around \emph{equilibria}---strategically stable outcomes. The underpinning of strategic stability can be naturally conceived in more ways than one, but none is more standard than the one pregnant in Nash's theorem~\citep{Nash51:Non}, which single-handedly propelled the rapid development of noncooperative game theory in the last century. This is reflected in, for example, the exulting words of the Nobel prize-winning economist Roger Myerson, who juxtaposed the ``fundamental and pervasive impact of Nash equilibrium'' in economics and the social sciences to ``the discovery of the DNA double helix in the biological sciences~\citep{Myerson99:Nash}.'' This profound influence nonwithstanding, the concept of Nash equilibrium has also been subjected to fierce criticism. Some of the most vocal critiques have been articulated in the computer science community, vehemently objecting to its apparent intractability~\citep{Daskalakis08:Complexity,Chen09:Settling,Rubinstein16:Settling,Rubinstein15:Inapproximability,Etessami07:Complexity}.

Aumann's pathbreaking \emph{correlated equilibrium} concept~\citep{Aumann74:Subjectivity} offers a compelling alternative that promises to alleviate such concerns: unlike Nash equilibria, which amount to fixed points of general functions, correlated equilibria can be expressed as linear programs. This, in turn, enables their efficient computation, or at least so the common narrative goes. The reality is more nuanced, and is dictated by the underlying \emph{game representation}. In particular, strategic interactions encountered in practice often unfold sequentially. Such settings do not lend themselves to the usual one-shot (aka. normal-form) representation---not without blowing up the description of the game, that is. Instead, sequential games are usually represented in \emph{extensive form} (or other pertinent paradigms~\citep{Littman94:Markov}). In stark contrast to normal-form games, wherein polynomial-time algorithms have long been established~\citep{Papadimitriou08:Computing,Jiang11:Polynomial}, the complexity of computing correlated equilibria in extensive-form games remains an enigmatic open problem.

\citet{Daskalakis24:Lower} recently provided evidence in the negative through the prism of the \emph{no-regret} framework. In particular, there is an intimate nexus between correlated equilibria and online learning: learners minimizing \emph{swap regret}---a powerful notion of hindsight rationality---converge to the set of correlated equilibria~\citep{Blum07:From,Stoltz05:Internal}. In this context, \citet{Daskalakis24:Lower} established an exponential lower bound for the number of iterations needed so that a learner incurs at most $\epsilon$ average swap regret when facing an adversary; more precisely, that result applies in the regime where $\epsilon$ is inversely polynomial (\emph{cf.} \citet{Dagan24:From,Peng24:Fast}). While this result does not rule out the existence of efficient algorithms beyond the adversarial regime, it does immediately bring to the fore a well-studied but pressing question: \emph{what notions of hindsight rationality are efficiently learnable?}

Hindsight rationality in online learning can be understood through a set of functions, $\Phi$, so that no deviation according to a function in $\Phi$ can retrospectively improve the cumulative utility; such a learner is said to be consistent with minimizing \emph{$\Phi$-regret}~\citep{Greenwald03:General,Stoltz07:Learning,Gordon08:No}. The broader the set of deviations $\Phi$, the more appealing the ensuing concept of hindsight rationality. The usual notion of external regret is an instantiation of that framework for which $\Phi$ contains solely constant functions---referred to as \emph{coarse} deviations. On the other end of the spectrum, when $\Phi$ contains all possible deviations, one finds the powerful notion of swap regret---associated with (normal-form) correlated equilibria. The fundamental question thus is to characterize the structure of $\Phi$ that enables efficient learnability---and, indeed, computation.

Much of the recent research in the context of learning in extensive-form games has focused on this exact problem. This can be traced back to the influential work of~\citet{Zinkevich07:Regret}, who introduced \emph{counterfactual regret minimization (CFR)}---an algorithm that was at the heart of recent landmark results in AI benchmarks such as poker~\citep{Brown17:Superhuman,Brown19:Superhuman,Bowling15:Heads,Moravvcik17:DeepStack}. CFR is an online algorithm for minimizing external regret---associated with (normal-form) coarse correlated equilibria. Moving forward, efficient algorithms eventually emerged for \emph{extensive-form correlated equilibria (EFCE)}~\citep{Huang08:Computing,Farina22:Simple,Dudik09:SamplingBased,Bai22:Efficient} (\emph{cf.}~\citet{Morrill21:Efficient,Morrill21:Hindsight}), and more broadly, when $\Phi$ contains solely \emph{linear} functions~\citep{Farina24:Polynomial,Farina23:Polynomial}---corresponding to \emph{linear correlated equilibria (LCE)}. \citet{Daskalakis24:Efficient} recently took a step even further by strengthening those results whenever the underlying constraint set $\cX \subseteq \R^d$ admits a membership oracle. \Cref{fig:taxonomy} summarizes the landscape that has emerged.

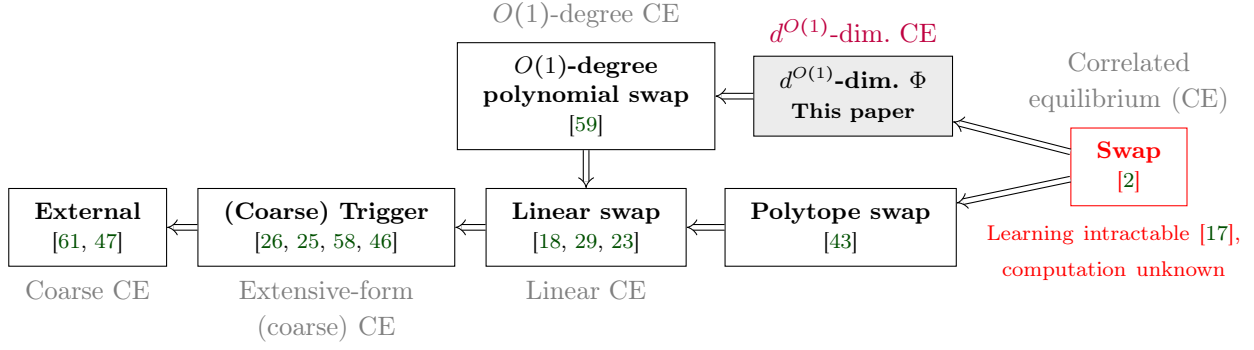
\begin{figure}[t]
\tikzstyle{box} = [rectangle, minimum width=.5cm, minimum height=0.5cm, text centered, draw=black, font=\footnotesize]
\tikzstyle{textbelow} = [font=\scriptsize, align=center]
\tikzstyle{arrow} = [implies-,double equal sign distance] 

\begin{tikzpicture}[node distance=0.6cm]
    \node (NFCCE) [box] {\begin{tabular}{c} \textbf{External} \\ \scriptsize \citep{Zinkevich07:Regret,Moulin78:Strategically} \end{tabular}};
    \node (EFCE) [box, right=0.4cm of NFCCE] {\begin{tabular}{c} \textbf{(Coarse) Trigger} \\ \scriptsize \citep{Farina22:Simple,Farina20:Coarse,Stengel08:Extensive,Morrill21:Hindsight} \end{tabular}};
    \node (LCE) [box, right=0.4cm of EFCE] {\begin{tabular}{c} \textbf{Linear swap} \\ \scriptsize \citep{Daskalakis24:Efficient,Gordon08:No,Farina23:Polynomial} \end{tabular}};
    
    \node (LowDegree) [box, above=.5cm of LCE] {\begin{tabular}{c} \textbf{$O(1)$-degree} \\ \textbf{polynomial swap} \\ \scriptsize \citep{Zhang24:Efficient} \end{tabular}};
    \node (Polytope) [box, right=0.5cm of LCE] {\begin{tabular}{c} \textbf{Polytope swap} \\ \scriptsize \citep{Mansour22:Strategizing} \end{tabular}};
    \node (LowDim) [box, fill=gray!15, right=.5cm of LowDegree] {\begin{tabular}{c} \textbf{$d^{O(1)}$-dim. $\Phi$} \\ \scriptsize \textbf{This paper} \end{tabular}};
    \node (NFCE) [box, red, right=1.5cm of Polytope,yshift=8mm] {\begin{tabular}{c} \textbf{Swap} \\ \scriptsize \citep{Aumann74:Subjectivity} \end{tabular}};

    \node[red, below=1mm of NFCE, text width=3.4cm,align=center,xshift=-2mm] {\scriptsize Learning intractable \citep{Daskalakis24:Lower}, computation unknown};
    \node[below=.5mm of NFCCE,gray] {\small Coarse CE};
    \node[below=.5mm of EFCE,text width=2.4cm,align=center,gray] {\small Extensive-form (coarse) CE};
    \node[below=.5mm of LCE,gray] {\small Linear CE};
    \node[above=.5mm of LowDegree,gray] {\small $O(1)$-degree CE};
    \node[above=.5mm of LowDim,purple] {\small $d^{O(1)}$-dim. CE};
    \node[above=.5mm of NFCE,text width=2.8cm,align=center,gray] {\small Correlated equilibrium (CE)};

    \path
     (NFCCE) edge[arrow] (EFCE)
     (EFCE) edge[arrow] (LCE)
     (LCE) edge[arrow] (LowDegree)
     (LowDegree) edge[arrow] (LowDim)
     (LCE) edge[arrow] (Polytope)
     (LowDim) edge[arrow] (NFCE)
     (Polytope) edge[arrow] (NFCE);


\end{tikzpicture}
\caption{The arrows $A \implies B$ denote that minimizing the notion of regret $A$ implies minimizing the notion of regret $B$. In other words, $A$ defines a superset of deviations that the learner considers compared to $B$, and hence leads to a stronger notion of equilibrium. The gray text below or above a notion of regret denotes the name of the corresponding notion of equilibrium, if applicable.}
\label{fig:taxonomy}
\end{figure}

\subsection{Our results: \texorpdfstring{$\Phi$}{Phi}-equilibria at the frontier of tractability}
\label{sec:results}

\begin{table}[htbp]
\centering
\small
\caption{Our main results for the $k$-dimensional set $\Phi^m$ of~\Cref{assumption:kernel}; $k \gg d$.}
\vspace{.3cm}
\label{tab:bounds}
\begin{tabular}{lcc}
\toprule
 & Upper bound & Lower bound \\
\midrule
\multirow{2}{*}{\textbf{Learning}} & $\poly(k)/\epsilon^2$ & $\min \{ \sqrt{k}/4, \text{exp}(\Omega(\epsilon^{-1/6})) \}$ \\
 & (\Cref{theorem:main1}) & (\Cref{theorem:mainlower}) \\
\midrule
\multirow{2}{*}{\textbf{Computation}} & $\poly(k, \log(1/\epsilon))$ & \multirow{2}{*}{Open question} \\
 & (\Cref{theorem:main-eah}) & \\
\bottomrule
\end{tabular}
\end{table}

The primary focus of this paper is to expand the scope of that prior research beyond linear-swap regret---associated with linear correlated equilibria---toward the frontier of tractability. While handling all swap deviations is impossible in light of the recent lower bound of~\citet{Daskalakis24:Lower}, discussed above, we are able to cope with the broad class of functions introduced below.

\begin{definition}\label{assumption:kernel}
    Given a map $m : \cX \to \R^{k'}$, the set of deviations $\Phi^m$ is defined as the set of all maps $\phi : \cX \to \cX$ that can be can be expressed by the matrix-vector product $\mK(\phi ) m(\vx) + \cons(\phi)$ for some $\mK(\phi) \in \R^{d \times k'}$ and $\cons(\phi) \in \R^d$. 
The set of functions $\Phi^m$ has dimension at most $k \defeq k' \cdot d + d$.
\end{definition}

We think of $k$ as a measure of the complexity of $\Phi^m$; in what follows, one may imagine $k \leq \poly(d)$. There is a clear sense in which going beyond~\Cref{assumption:kernel} is daunting: even representing such functions becomes prohibitive. Indeed, we also establish lower bounds that preclude going beyond the set of deviations in~\Cref{assumption:kernel}, showing that our results cannot be significantly improved.

As a canonical example, one can capture degree-$\ell$ polynomials by taking $m(\vx)$ to be the function that outputs all $\ell$-wise (and lower) products of entries in $\vx$ (hence $k = d^{O(\ell)}$), and $\mK$ the matrix of coefficients of the polynomial. (For technical reasons, we actually consider a certain orthonormal basis for polynomials introduced formally in~\Cref{def:polys}.)

\subsubsection{Learning}

We begin by stating our results in the usual no-regret framework in online learning, and we then proceed with the centralized model of computation. A key reference point here is the recent paper of~\citet{Zhang24:Efficient}, who provided online algorithms in extensive-form games when $\Phi$ contains low-degree polynomials. However, the complexity of their algorithm---both the per-iteration running time and the total number of rounds---depends on the depth of the game tree; in particular, for general extensive-form games, their bounds become vacuous even for constant-degree polynomials (for example, when the game tree is lopsided). Our first main result addresses that limitation. What is more, it encompasses a more general class than extensive-form decision problems, while at the same time going beyond low-degree swap deviations.

\begin{theorem}[Online learning; precise version in~\Cref{theorem:main-prec}]
    \label{theorem:main1}
    Suppose that $\cX \subseteq \R^d$ admits a membership oracle and $\Phi^m$ is $k$-dimensional per~\Cref{assumption:kernel}. There is an online algorithm that guarantees at most $\epsilon$ average $\Phi^m$-regret after $\poly(k) / \epsilon^2$ rounds with $\poly(k, 1/\epsilon)$ running time.
\end{theorem}

\subsubsection{Computation}

The result above holds even when the learner is facing an adversary, thereby being readily applicable when learning in $n$-player \emph{mutlilinear games}. In such games, each player $i \in \range{n}$ has a convex and compact strategy set $\cX_i \subseteq \R^{d_i}$ and utility function $u_i : \cX_1 \times \dots \times \cX_n \to \R$ that is linear in $\cX_i$, so that $u_i(\vx) = \langle \vec{g}_i, \vx_i \rangle$ for some $\vec{g}_i = \vec{g}_i(\vx_{-i}) \in \R^{d_i}$. (Extensive-form games constitute a canonical example of this framework.) In this context, \Cref{theorem:main1} implies a fully polynomial-time algorithm ($\FPTAS$) for computing $\epsilon$-approximate $\Phi^m$-equilibria in convex games. Our next result establishes a polynomial-time algorithm---that is, with running time growing polynomially in $\log(1/\epsilon)$ as opposed to $1/\epsilon$---for that problem.

\begin{theorem}[Computation; precise version in~\Cref{theorem:main-eah-prec}]
    \label{theorem:main-eah}
    Consider an $n$-player multilinear game $\Gamma$ such that, for each player $i \in [n]$, we are given $\poly(n, k)$-time algorithms for the following:
\begin{itemize}
\item an oracle to compute the gradient, that is, the vector $\vec{g}_i = \vec{g}_i(\vx_{-i}) \in \R^{d_i}$ for which $\ip{\vg_i(\vx_{-i}), \vx_i} = u_i(\vx)$ for all $\vx \in \cX_1 \times \dots \times \cX_n$ (polynomial expectation property); and
\item a membership oracle for the strategy set $\cX_i$.
\end{itemize}
Suppose further that each $\Phi^{m_i}$ is $k_i$-dimensional per~\Cref{assumption:kernel} and $\|\vec{g}_i \| \leq B$. Then, an $\eps$-approximate $\Phi^m$-equilibrium of $\Gamma$ can be computed in $\poly\qty(n, k, \log(B/\eps))$ time.
\end{theorem}

\subsection{Technical approach}

\Cref{theorem:main1,theorem:main-eah} build on and extend certain recent developments due to~\citet{Daskalakis24:Efficient} and~\citet{Zhang24:Efficient}. Below, we outline our key technical contributions.

\paragraph{Expected fixed points} Our approach crucially hinges on the notion of an \emph{expected fixed point}: a distribution $\mu \in \Delta(\cX)$ such that $\E_{\vx \sim \mu} [ \phi(\vx) - \vx] \approx 0$. Taking a step back, earlier approaches for minimizing $\Phi$-regret were based on computing an actual fixed point of a function $\phi \in \Phi$~\citep{Gordon08:No,Blum07:From,Stoltz05:Internal}. For normal-form games, computing a fixed point of such a function boils down to determining the stationary distribution of a certain Markov chain, which is directly amenable to linear programming---this holds more generally when $\Phi$ contains linear functions. However, when considering nonlinear functions, this standard approach immediately becomes prohibitive since fixed points are hard to compute. \citet{Zhang24:Efficient} bypass this obstacle by showing that a fixed point \emph{in expectation}, as introduced earlier, is actually sufficient for minimizing $\Phi$-regret. In fact, building on an earlier result due to~\citet{Hazan07:Computational}, \citet{Zhang24:Efficient} observed that approximating expected fixed points also reduces to minimizing $\Phi$-regret, establishing a certain equivalence between the two. Crucially, there is a simple, $O(1/\epsilon)$-time algorithm for computing an $\epsilon$-expected fixed point of any $\phi \in \Phi$ by taking the uniform distribution over the sequence $\vx, \phi(\vx), \phi(\phi(\vx)), \dots$; the claimed guarantee follows by a telescopic summation. In this context, an important question left open by~\citet{Zhang24:Efficient}---which, as we shall see, is the crux in proving~\Cref{theorem:main-eah}---concerns the complexity of computing expected fixed points in the regime where $\epsilon$ is exponentially small. We address this question by showing that expected fixed points can be computed in time polynomial in the dimension and $\log(1/\epsilon)$.

\begin{theorem}[Expected fixed points]
    \label{theorem:efps}
    Given oracle access to $\cX$ and $\phi : \cX \to \cX$, there is a $\poly(d, \log(1/\epsilon))$-time algorithm that computes an $\epsilon$-expected fixed point of $\phi$.
\end{theorem}

We have described so far the role of expected fixed points when learning in an online environment (\emph{cf.}~\Cref{theorem:main1}). Going back to~\Cref{theorem:main-eah}, expected fixed points also serve a crucial purpose in that context. Namely, \Cref{theorem:main-eah} relies on the \emph{ellipsoid against hope (\eah)} algorithm of~\citet{Papadimitriou08:Computing}, which in turn is based on running the ellipsoid algorithm on an infeasible program---the rationale being that a correlated equilibrium can be extracted from the certificate of infeasibility of that program. Now, to execute ellipsoid, one needs a separation oracle. For normal-form games, this amounts to a fixed point oracle: for any $\phi$, compute $\vx \in \cX$ such that $\phi(\vx) = \vx$. And, as we saw earlier, $\phi$ is a just a stochastic matrix, and so it suffices to identify a stationary distribution of the corresponding Markov chain.

However, there are two main obstacles, which manifest themselves in each iteration of the ellipsoid, when $\cX$ is a general convex set and $\Phi$ is allowed to contain nonlinear endomorphisms:

\begin{enumerate}
    \item computing a fixed point of a nonlinear $\phi$ is intractable; and\label{item:first}
    \item separating over the set $\Phi$ is also intractable even with respect to linear endomorphisms~\citep{Daskalakis24:Efficient}, let alone under~\Cref{assumption:kernel}.\label{item:second}
\end{enumerate}

With regard to~\Cref{item:first}, we show that, during the execution of the ellipsoid, one might as well use \emph{expected} fixed points, which are tractable by virtue of~\Cref{theorem:efps} we described earlier. What is intriguing is that our proof of~\Cref{theorem:efps} also relies on (a different instantiation of) $\eah$, and so the overall algorithm that we develop uses $\eah$ in a nested fashion---each separation oracle as part of the execution of $\eah$ is internally implemented via $\eah$!

For~\Cref{item:second}, we build on the framework of~\citet{Daskalakis24:Efficient}. In light of the inability to efficiently separate over linear endomorphisms, they observed that one can still execute $\eah$ with access to a weaker oracle, which they refer to as a \emph{semi-separation oracle}. Moreover, they developed a polynomial-time semi-separation oracle with respect to the set of linear endomorphisms. Building on~\Cref{theorem:efps}, we significantly extend their result, establishing a semi-separation oracle for general functions, not just linear ones.

\begin{theorem}[Semi-separation oracle for general functions]
    Given oracle access to $\cX$ and $\phi : \cX \to \R^d$, there is a $\poly(d, \log(1/\epsilon))$-time algorithm that either computes an $\epsilon$-expected fixed point of $\phi$, or identifies a point $\vx \in \cX$ such that $\phi(\vx) \notin \cX$.
\end{theorem}

(In particular, in the usual case where $\phi$ maps to $\cX$, the algorithm above always returns an $\epsilon$-expected fixed point of $\phi$.)

We now turn to~\Cref{theorem:main1}, which revolves around the online learning setting. Equipped with our semi-separation oracle for general functions, we show that the framework of~\citet{Daskalakis24:Efficient} can be extended from linear endomorphisms to ones satisfying~\Cref{assumption:kernel}. The technical pieces underpinning~\Cref{theorem:main1} are exposed in depth in~\Cref{sec:reg}. Importantly, the dimension of $\Phi$ turns out to be a fundamental barrier for no-regret learning in the following sense.

\begin{restatable}{theorem}{lowerbound}
    \label{theorem:mainlower}
    For any $k$ and any $d \ge \Theta(\log^{14} k)$, there is an online decision problem with dimension $d$ and an adversary such that the $\Phi$-regret of the learner with respect to a $k$-dimensional $\Phi$ is at least $\epsilon$ when $T < \min \{ \sqrt{k}/4, \exp(\Omega(\epsilon^{-1/6})) \}$.
\end{restatable}

This is a straightforward refinement of recent lower bounds~\citep{Daskalakis24:Lower,Dagan24:From,Peng24:Fast}. What is important is that, up to constant factors in the exponent of $k$, \Cref{theorem:mainlower} matches the upper bound of~\Cref{theorem:main1}. In doing so, we establish for the first time a class of deviations that characterizes---in the previous sense---no-$\Phi$-regret learning in the adversarial setting.

Finally, we remark that we did not attempt to optimize the polynomial dependencies on $d$ and $k$ throughout this paper; improving the overall complexity is an interesting direction for future work.

\subsection{Further related work}
\label{sec:related}

The existence of no-regret algorithms goes back to the pioneering work of~\citet{Blackwell56:analog}; the stronger notion of swap regret was crystallized and analyzed more recently~\citep{Blum07:From,Stoltz05:Internal,Hart00:Simple}. Part of the impetus of that line of work revolves around the connection to correlated equilibria, highlighted earlier. Unfortunately, beyond online decision problems on the simplex, such no-swap-regret algorithms become inefficient when the number of pure strategies is exponential in the natural parameters of the problem---as is the case, for example, in Bayesian games, wherein $\cX \defeq [0, 1]^d$. Recent breakthrough results by~\citet{Dagan24:From} and~\citet{Peng24:Fast} establish a new algorithmic paradigm for minimizing swap regret, applicable even when the number of pure strategies is exponential. However, that comes at the expense of introducing an exponential dependence on $1/\epsilon$, which is unavoidable in the adversarial regime~\citep{Daskalakis24:Lower}. Our main interest here is in online algorithms with complexity scaling polynomially in both the dimension and $1/\epsilon$.

Besides the game-theoretic implications concerning convergence to correlated equilibria, swap regret is a fundamental concept in its own right, being intimately tied to the notion of \emph{calibration}; namely, it has been known since the foundational work of~\citet{Foster97:Calibrated} that best-responding to calibrated forecasters guarantees no-swap-regret (\emph{cf.}\  \citet{Foster18:Smooth}); in relation to that connection, it is worth noting an important, recent body of work that bypasses the intractability of calibration in high dimensions~\citep{Noarov23:High,Roth24:Forecasting,Hu24:Predict}. Swap regret is also more robust against exploitation, in a sense formalized in a series of recent papers~\citep{Assos24:Maximizing,Mansour22:Strategizing,Deng19:Strategizing,Guruganesh24:Contracting}. 

In particular, within that line of work, \citet{Mansour22:Strategizing} introduced the notion of \emph{polytope swap regret}, which comprises deviations that allow the vertices of the underlying polytope to be swapped with each other---points inside the polytope are mapped in accordance with the (worst-case) convex combination of vertices. It is currently unknown whether there is an efficient algorithm for minimizing polytope swap regret.

The more flexible framework of $\Phi$-regret, which has been gaining considerable traction in recent years, allows one to circumvent the recent lower bound of~\citet{Daskalakis24:Lower} by restricting the set of deviations. In addition to the research highlighted above, chiefly in the context of extensive-form games, we now provide some further pointers for the interested reader. \citet{Bernasconi23:Constrained} considered the more challenging setting of so-called ``pseudo-games,'' wherein players have joint constraint sets. $\Phi$-equilibria in such settings have certain counterintuitive properties; for example, they are not necessarily convex. $\Phi$-equilibria have also garnered attention in the context of Markov (aka.\ stochastic) games, going back to the work of~\citet{Greenwald03:Correlated}; for more recent research, we refer to~\citet{Jin24:Vlearning,Erez23:Regret,Cai24:Near}, and references therein. Even more broadly, we refer to~\citet{Cai24:Tractable,Ahunbay25:First} for efficient solution concepts in nonconcave games~\citep{Daskalakis22:Non}.

Finally, we remark that the hardness result of~\citet{Daskalakis24:Efficient} for separating over linear endomorphisms does not apply to polytopes represented with a polynomial number of constraints. Indeed, it is relatively straightforward to implement a membership oracle for such polytopes~\citep{Daskalakis24:Efficient}. In contrast, it is computationally hard to decide membership for low-degree polynomials~\citep{Zhang24:Efficient}.
\section{Preliminaries}
\label{sec:prel}

\paragraph{Notation} We use boldface lowercase letters, such as $\vx$ and $\vy$, to denote vectors in a Euclidean space. Matrices are represented with capital boldface letters, such as $\mK$. For a vector $\vx \in \R^d$, we denote by $\|\vx \| \defeq \sqrt{\langle \vx, \vx \rangle}$ its Euclidean norm, where $\langle \cdot, \cdot \rangle$ is the inner product. The $j$th coordinate of $\vx$ is accessed by $\vx[j]$. For a matrix $\mK$, we use $\|\mK\|_F$ to denote its Frobenius norm. $\mI_{d \times d} \in \R^{d \times d}$ represents the identity matrix. $\cB_r(\vec{x})$ is the (closed) Euclidean ball centered at $\vx$ with radius $r > 0$. An \emph{endomorphism} on $\cX$ is a function mapping $\cX$ to $\cX$.

\paragraph{Oracle access} Throughout this paper, we assume that we have access to the convex and compact constraint set $\cX$ via an oracle. (For multi-player games, oracle access is posited for the constraint set $\cX_i$ of each player, which can be thereby extended to $\cX \defeq \cX_1 \times \dots \times \cX_n$.) In particular, the following three types of oracles are commonly considered in the literature:
\begin{itemize}
    \item \emph{membership}: given a point $\vx \in \R^d$, decide whether $\vx \in \cX$;
    \item \emph{separation}: given a point $\vx \in \R^d$, decide whether $\vx \in \cX$, and if not, output a hyperplane $\vec{w} \in \R^d$ separating $\vx$ from $\cX$: $\langle \vx, \vec{w} \rangle > \langle \vx', \vec{w} \rangle$ for all $\vx' \in \cX$;
    \item \emph{linear optimization}: given a point $\vec{u} \in \R^d$, output any point in $\argmax_{\vx \in \cX} \langle \vx, \vu \rangle$.
\end{itemize}
Under the assumption that $\cB_r(\cdot) \subseteq \cX \subseteq \cB_R(\vec{0})$, the three oracles described above are known to be (polynomially) equivalent~\citep{Grotschel93:Geometric}---up to logarithmic factors in $R$ and $1/r$. The previous geometric condition can always be met by bringing $\cX$ into \emph{isotropic position}, which means that, for a uniformly sampled $\vx \sim \cX$, we have $\E[\vx] = 0$ and $\E[ \vx \vx^\top] = \mI_{d \times d}$. This can be achieved in polynomial time through an affine transformation~\citep{Lovasz06:Simulated,Bourgain96:Random,Kannan97:Random}; it is easy to see that minimizing $\Phi$-regret after applying that transformation suffices in order to minimize $\Phi$-regret in the original space (formally shown in~\Cref{sec:isotropic}). As a result, we can assume throughout that, for example, $\cB_1(\vec{0}) \subseteq \cX \subseteq \cB_{d}(\vec{0})$~\citep{Lovasz06:Simulated}.

\begin{remark}[Weak oracles]
    When dealing with general convex sets, the oracles posited above can return points supported on irrational numbers. To address this issue in the usual Turing model of computation, it suffices to consider weaker versions of those oracles that allow for some small slackness $\epsilon > 0$. Our analysis in the sequel can be extended to account for such imprecision. 
\end{remark}

\subsection{Online learning and $\Phi$-regret}
\label{sec:gordon}

In the usual framework of online learning, a learner interacts with an environment over a sequence of $T \in \N$ rounds. In each round $t \in [T]$, the learner selects a strategy $\vx^{(t)} \in \cX$, and then observes as feedback from the environment a utility function $\vx \mapsto \langle \vx, \vu^{(t)} \rangle$, for some utility vector $\vu^{(t)} \in [-1, 1]^{d}$; the utility of the learner in the $t$th round is given by $\langle \vx^{(t)}, \vu^{(t)} \rangle$. For the purpose of our work, we will allow the learner to output a \emph{mixed} strategy, $\mu^{(t)} \in \Delta(\cX)$, so that the (expected) utility at the $t$th round reads $\E_{ \vx^{(t)} \sim \mu^{(t)}} \langle \vx^{(t)}, \vu^{(t)} \rangle $. As shown by~\citet{Zhang24:Efficient}, restricting the learner to output strategies in $\cX$---as opposed to $\Delta(\cX)$---makes the problem of minimizing $\Phi$-regret \PPAD-hard, even with respect to low-degree polynomials, and so employing mixed strategies will be essential for our purposes.

In this context, a canonical measure of performance in online learning is \emph{$\Phi$-regret}, defined as
\begin{equation}
     \phireg^{(T)} \defeq \sup_{\phi \in \Phi} \sum_{t=1}^T \left\langle \vu^{(t)}, \E_{\vx^{(t)} \sim \mu^{(t)}} [ \phi(\vx^{(t)}) - \vx^{(t)} ] \right\rangle.
\end{equation}
The \emph{average} $\Phi$-regret is defined as $\frac{1}{T} \phireg^{(T)}$. Perhaps the most common instantiation of $\Phi$-regret is \emph{external} regret, whereby $\Phi$ contains solely constant transformations. We are interested in characterizing the broadest set of deviations $\Phi$ that allows for efficient learnability. Our starting point is the template of~\citet{Gordon08:No}.

\paragraph{The algorithm of~\citet{Gordon08:No}} \citet{Gordon08:No} (\emph{cf.}~\citet{Blum07:From,Stoltz05:Internal}) crystallized a basic template for minimizing $\Phi$-regret. It comprises two basic components:
\begin{enumerate}
    \item a fixed-point oracle $\fp(\phi)$ that takes as input any transformation $\Phi \ni \phi: \cX \to \cX$ and outputs a fixed point thereof; that is, a point $\vx \in \cX$ such that $\phi(\vx) = \vx$.\label{item:gordon1}
    \item an \emph{external} regret minimizer $\regbox_{\Phi}$ operating over the set $\Phi$.\label{item:gordon2}
\end{enumerate}
With access to the above components, a $\Phi$-regret minimizer $\regbox$ operating over $\cX$---without the need to resort to mixed strategies---can be constructed as follows. At any time $t \in [T]$, upon selecting a strategy $\vx^{(t)} \in \cX$ and observing $\vu^{(t)}$, provide as input to $\regbox_\Phi$ the utility function $\phi \mapsto \langle \vu^{(t)}, \phi(\vx^{(t)}) \rangle$. Suppose that $\phi^{(t+1)} \in \Phi$ is the next strategy of $\regbox_\Phi$. The learner $\regbox$ can then output as its next strategy $\vx^{(t+1)}$ any fixed point of $\phi^{(t+1)}$; that is, $\vx^{(t+1)} \defeq \fp(\phi^{(t+1)})$. By definition, it follows that the $\Phi$-regret of $\regbox$ is equal to the external regret of $\regbox_\Phi$ (\emph{cf.}~\Cref{theorem:gordon}).

However, our main result in the online learning setting hinges on relaxing both oracles posited in~\Cref{item:gordon1,item:gordon2} in the framework of~\citet{Gordon08:No}. With regard to~\Cref{item:gordon1}, when operating over \emph{mixed} strategies, \citet{Zhang24:Efficient} observed that it suffices to output an \emph{$\epsilon$-expected fixed point (EFP) of $\phi$}, that is, a distribution $\mu$ such that $\| \E_{\vx \sim \mu} [ \phi(\vx) - \vx ]  \|_1 \leq \epsilon$. Unlike actual fixed points, which are marred by computational intractability, \citet{Zhang24:Efficient} observed that there is a simple, $O(1/\epsilon)$-time algorithm for computing an $\epsilon$-expected fixed point: simply take the uniform distribution over the sequence $\vx, \phi(\vx), \phi(\phi(\vx)), \dots$ for $O(1/\epsilon)$ steps. (In fact, one of our main results---namely,~\Cref{theorem:efps}---provides a polynomial-time algorithm for that problem.) The overall scheme resulting from replacing~\Cref{item:gordon1} with an approximate EFP is given in~\Cref{alg:gordon}.

\begin{algorithm}[!ht]
\caption{Minimizing $\Phi$-regret with EFPs~\citep{Gordon08:No,Zhang24:Efficient}}
\label{alg:gordon}
\SetKwInOut{Input}{Input}
\SetKwInOut{Output}{Output}
\SetKw{Input}{Input:}
\SetKw{Output}{Output:}
\Input{
\begin{itemize}[noitemsep,topsep=0pt]
    \item An external regret minimizer $\regbox_\Phi$ for the set $\Phi$
    \item A convex and compact strategy set $\cX$
    \item Precision parameter $\epsilon > 0$
\end{itemize}
}
\Output{An $\Phi$-regret minimizer $\regbox$ for the set $\cX$}\;
Initialize $\phi^{(1)} \in \Phi$\;
 \For{$t=1, \dots, T$}{
    Set $\mu^{(t)} \in \Delta(\cX)$ to be an $\epsilon$-expected fixed point of $\phi^{(t)}$\;
     Output $\mu^{(t)} \in \Delta(\cX)$ as the next mixed strategy of $\regbox$\;
     Receive as feedback $\vu^{(t)}$\;
     Give as input to $\regbox_\Phi$ the utility function $\phi \mapsto \E_{\vx^{(t)} \sim \mu^{(t)}} \langle \phi(\vx^{(t)}), \vu^{(t)} \rangle$\;
     Let $\phi^{(t+1)} \in \Phi$ be the next strategy of $\regbox_\Phi$\;
 }
\end{algorithm}

\begin{theorem}[\citep{Gordon08:No,Zhang24:Efficient}]
    \label{theorem:gordon}
    Let $\phireg^{(T)}$ be the $\Phi$-regret of $\regbox$ and $\reg_{\Phi}^{(T)}$ the external regret of $\regbox_\Phi$ in~\Cref{alg:gordon} with precision $\epsilon > 0$. Then, for any $T \in \N$,
    \begin{equation}
        \phireg^{(T)} \leq \reg_{\Phi}^{(T)} +  \epsilon T.
    \end{equation}
\end{theorem}

In particular, taking, say, $\epsilon \propto \nicefrac{1}{\sqrt{T}}$, \Cref{theorem:gordon} reduces $\Phi$-regret minimization on $\cX$ to external regret minimization on $\Phi$.

\begin{proof}[Proof of~\Cref{theorem:gordon}]
    For any $\phi \in \Phi$, we have
    \begin{equation}
        \label{eq:Phi-bound}
        \sum_{t=1}^T \left\langle \vu^{(t)}, \E_{\vx^{(t)} \sim \mu^{(t)}} [ \phi(\vx^{(t)}) - \vx^{(t)} ] \right\rangle \leq \sum_{t=1}^T \E_{\vx^{(t)} \sim \mu^{(t)}} \langle \phi(\vx^{(t)}) - \phi^{(t)}(\vx^{(t)}), \vu^{(t)} \rangle + \epsilon T,
    \end{equation}
    where we used the fact that $\mu^{(t)}$ is an $\epsilon$-expected fixed point of $\phi^{(t)}$ for all $t \in [T]$. The right-hand side of~\eqref{eq:Phi-bound} can be in turn bounded by the external regret of $\regbox_\Phi$ plus the slackness term $\epsilon T$.
\end{proof}

This paradigm for minimizing $\Phi$-regret has been ubiquitous in prior work in this area. And yet,~\citet{Daskalakis24:Efficient} recently demonstrated that it is insufficient even when $\Phi$ contains all linear endomorphisms of a general convex set. \Cref{sec:reg} covers in detail the framework of~\citet{Daskalakis24:Efficient}---relaxing~\Cref{item:gordon2} of~\citet{Gordon08:No}---that will be the basis for our approach as well.

\paragraph{$\Phi$-equilibria} As we highlighted earlier, there is a celebrated connection between $\Phi$-regret and the game-theoretic solution concept of (correlated) \emph{$\Phi$-equilibrium}. More precisely, in the context of multilinear games as introduced in~\Cref{sec:results}, we recall the following central definition~\citep{Stoltz07:Learning,Greenwald03:General}.

\begin{definition}
    An \emph{$\eps$-approximate $\Phi$-equilibrium} of an $n$-player multilinear $\Gamma$ is a (correlated) distribution $\mu \in \Delta(\cX_1 \times \dots \times \cX_n)$ such that for every player $i \in [n]$ and deviation $\phi_i \in \Phi_i \subseteq \cX_i^{\cX_i}$,
\begin{align}
\E_{\vx \sim \mu} \qty[u_i(\phi_i(\vx_i), \vx_{-i}) - u_i(\vx)] \le \eps. \label{eq:phi-eqm}
\end{align}
\end{definition}

A direct consequence of this definition is that if players repeatedly interact in a game and all incur sublinear $\Phi$-regret, the average distribution of play converges to the set of $\Phi$-equilibria.

\subsection{Ellipsoid against hope}
\label{sec:eah}

The \emph{ellipsoid against hope ($\eah$)} algorithm was famously introduced by~\citet{Papadimitriou08:Computing} to compute correlated equilibria in succinct, multi-player games---under the polynomial expectation property. A further crucial assumption in the approach of~\citet{Papadimitriou08:Computing} is that the game is of \emph{polynomial type}, in that the number of actions (or pure strategies) is polynomial in the representation of the game. In contrast to normal-form games, extensive-form games---and many other natural classes of games---are \emph{not} of polynomial type. \citet{Farina24:Polynomial} recently showed how to apply~$\eah$ in the context of extensive-form games---albeit only for LCE; as we have seen, the complexity of NFCE remains open. We begin by recalling their framework, which crystallizes the approach of~\citet{Papadimitriou08:Computing}. We then proceed by introducing the more powerful approach of~\citet{Daskalakis24:Efficient}, which is crucial to compute LCE under general convex constraint sets, and which will form the basis for our approach as well.

Consider an arbitrary optimization problem of the form
\begin{equation}
    \label{eq:eah}
    \qq{find} \mu \in \Delta(\cX) \qq{s.t.} \E_{\vx \sim \mu} \ip{\vy, G(\vx)} \ge 0 \quad \forall \vy \in \cY,
\end{equation}
where $\cX \subseteq \R^d$, $\cY \subseteq \R^k$, and $G : \cX \to \R^k$ is a function such that $\| G(\vx) \| \leq B$ for all $\vx \in \cX$. The crux in~\eqref{eq:eah} lies in the fact that $\mu$ resides in a high-dimensional (indeed, an infinite-dimensional) space, making standard approaches of little use. $\eah$ addresses that challenge, as we describe next.

Suppose that we are given a $\poly(d, k)$-time evaluation oracle for $G$ and a separation oracle $(\sep)$ for $\cY$, assumed to be \emph{well-bounded}: $\cB_r(\cdot) \subseteq \cY \subseteq \cB_R(\vec{0})$. In addition, we assume that we have access to a {\em good-enough-response} ($\ger$) oracle, which, given any $\vy \in \cY$, returns $\vx \in \cX$ such that $\ip{\vy, G(\vx)} \ge 0$. The $\eah$ algorithm allows us to solve problems of the form \eqref{eq:eah} with just the above tools. In particular, $\eah$ proceeds by considering an $\eps$-approximate version of the dual of \eqref{eq:eah}.
\begin{align} \label{eq:eah-dual}
    \qq{find} \vy \in \cY \qq{s.t.} \ip{\vy, G(\vx)} \le -\eps \quad \forall \vx \in \cX.
\end{align}
Since a $\ger$ oracle exists, \eqref{eq:eah-dual} is infeasible. Moreover, a certificate of infeasibility of \eqref{eq:eah-dual} provides an $\eps$-approximate solution to \eqref{eq:eah}. Thus, it suffices to run the ellipsoid algorithm on \eqref{eq:eah-dual} and extract a certificate of infeasibility. This is precisely what $\eah$ does, as formalized in~\Cref{theorem:eah}; the overall scheme is~\Cref{alg:eah} in~\Cref{sec:aux} (\Cref{alg:gen-eah} below is a more general version thereof).

\begin{theorem}[Generalized form of $\eah$~\citep{Farina24:Polynomial,Papadimitriou08:Computing}]
    \label{theorem:eah}
    Suppose that we have $\poly(d, k)$-time algorithms for the following:
    \begin{itemize}[noitemsep,topsep=0pt]
        \item an evaluation oracle for $G$, where $\| G(\vx) \| \leq B$ for all $\vx \in \cX$; 
        \item a $\ger$ oracle for~\eqref{eq:eah}; and
        \item a separation oracle ($\sep$) for the well-bounded set $\cY$.
    \end{itemize}
    Then, there is an algorithm that runs in time $\poly(d, k, \log(B/\eps))$ and returns an $\eps$-approximate solution to \eqref{eq:eah}.
\end{theorem}

\begin{algorithm}[!ht]
\caption{Ellipsoid against hope ($\eah$) under $\either$ oracle~\citep{Daskalakis24:Efficient}}
\label{alg:gen-eah}
\SetKwInOut{Input}{Input}
\SetKwInOut{Output}{Output}
\SetKw{Input}{Input:}
\SetKw{Output}{Output:}
\Input{
    \begin{itemize}[noitemsep,topsep=0pt]
        \item Parameters $R_y, r_y > 0$ such that $\cB_{r_y}(\cdot) \subseteq \cY \subseteq \cB_{R_y}(\vec{0})$
        \item Precision parameter $\epsilon > 0$
        \item Parameter $B > 0$ such that $\| G(\vx) \| \leq B$ for all $\vx \in \cX$
        \item A $\either$ oracle (\Cref{def:either})
    \end{itemize}
}
\Output{A sparse, $\epsilon$-approximate solution $\mu \in \Delta(\cX)$ of~\eqref{eq:eah}}\;
Initialize the ellipsoid $\cE \defeq \cB_{R_y}(\vec{0})$\;
Initialize $\tilde{\cY} \defeq \cB_{R_y}(\vec{0})$\;
\While{$\vol(\cE) \geq \vol(\cB_{\epsilon/B}(\cdot))$}{
    Query the $\either$ oracle on the center of $\cE$\;
    \eIf{it returns a good-enough-response $\vx \in \cX$}{
        Update $\cE$ to the minimum volume ellipsoid containing $\cE \cap \{ \vy \in \R^k : \langle \vy, G(\vx) \rangle \leq 0 \}$\;
    }{
        Let $H$ be the halfspace that separates $\vy$ from $\cY$\;
        Update $\cE$ to the minimum volume ellipsoid containing $\cE \cap H$\;
        Update $\tilde{\cY} \defeq \tilde{\cY} \cap H$\;
    }
}
Let $\vx^{(1)}, \dots, \vx^{(T)}$ be the $\ger$ oracle responses produced in the process above\;
Define $\mG \defeq [G(\vx^{(1)}) \mid \hdots \mid G(\vx^{(T)})] \in \R^{k \times T} $\;
Compute a solution $\vec{\lambda}$ to the convex program
    $$\qq{find} \vec{\lambda} \in \Delta^T \qq{s.t.} \min_{\vy \in \tilde{\cY}} \vec{\lambda}^\top \mG^\top \vy \geq - \epsilon$$
\Return{$\Delta(\cX) \ni \mu \defeq \sum_{t=1}^T \lambda^{(t)} \mu(\vx^{(t)})$}
\end{algorithm}

However, when it comes to computing LCE under general constraint sets, \Cref{theorem:eah} is not enough: \citet{Daskalakis24:Efficient} showed that separating over $\cY$---the set of linear endomorphisms---is hard. In light of this fact, their key observation was that an $\epsilon$-approximate solution to~\eqref{eq:eah} can still be computed given access to a weaker oracle. Namely, instead of requiring both a $\ger$ and a $\sep$ oracle, as in~\Cref{theorem:eah}, \citet{Daskalakis24:Efficient} showed that it suffices to implement the following oracle: for any given $\vy \in \R^k$ (\emph{not} necessarily in $\cY$), 
\begin{enumerate}
    \item \emph{either} compute a good-enough response $\vx \in \cX$,\label{item:ger}
    \item \emph{or} a hyperplane separating $\vy$ from $\cY$.\label{item:sep}
\end{enumerate}
Although separating over $\cY$ is hard, this weaker oracle suffices to recover the guarantee of~\Cref{theorem:eah}, and this is enough to compute linear correlated equilibria in games. Yet, for our purposes, it will be necessary to relax the aforedescribed oracle even further, as formalized below.

\begin{definition}[$\either$]
    \label{def:either}
    Consider problem~\eqref{eq:eah}. The oracle $\epsilon$-$\either$ works as follows. It takes as input $\vy \in \R^k$, and it 
    \begin{enumerate}
        \item \emph{either} computes an $\epsilon$-approximate good-enough-response $\mu \in \Delta(\cX)$, $\E_{\vx \sim \mu} \langle \vy, G(\vx) \rangle \geq - \epsilon$, such that $\supp(\mu) \leq \poly(d, k, \log(1/\epsilon))$,
        \item \emph{or} a hyperplane $\epsilon$-approximately separating $\vy$ from $\cY$.
    \end{enumerate}
\end{definition}

Compared to the oracle described earlier (\Cref{item:ger,item:sep}), \Cref{def:either} makes two further concessions: first, the good-enough-response can now be a distribution, so long as it has polynomial support; and second, both $\ger$ and $\sep$ can have some small slack $\epsilon > 0$. Both of those relaxations will be essential for our applications. We now summarize the key guarantee.

\begin{theorem}[\citep{Daskalakis24:Efficient}; generalization of~\Cref{theorem:eah}]
    \label{theorem:either}
    Suppose that we have $\poly(d, k, \log(1/\epsilon))$-time algorithms for the following:
    \begin{itemize}[noitemsep,topsep=0pt]
        \item an $\epsilon$-$\either$ oracle with respect to the well-bounded set $\cY$, and
        \item an evaluation oracle for $G$, where $\| G(\vx) \| \leq B$ for all $\vx \in \cX$.
    \end{itemize}
    Then, there is an algorithm that runs in time $\poly(d, k, \log(B/\epsilon))$ and returns an $\epsilon$-approximate solution to~\eqref{eq:eah}.
\end{theorem}

\Cref{alg:gen-eah} depicts the overall scheme under a $\either$ oracle. (The last line of the algorithm uses the notation $\mu(\vx)$ for the distribution supported solely on $\vx \in \cX$.)
\section{Sets of deviations with polynomial dimension}
\label{sec:deviations}

In this section, we formally introduce the assumptions we make concerning the feature map $m$ of~\Cref{assumption:kernel}, and we then provide a canonical example that satisfies our blanket assumptions.

\begin{assumption}
    \label{assumption:kernel-precise}
    We make the following assumptions regarding $\Phi^m$ and $m$ of~\Cref{assumption:kernel}:
    \begin{itemize}
        \item $m : \cX \to \R^{k'}$ is computable in $\poly(k)$ time.
        \item $\co m(\cB_1(\vec{0})) \supseteq \cB_\delta(\vec{0})$ for some $\delta \geq \poly(1/k)$.
        \item $\|m(\vx)\| \leq \poly(k)$ for all $\vx \in \cX$, and $m(\vec{0}) = \vec{0}$. 
        \item $\Phi^m$ contains the identity map.
    \end{itemize}
\end{assumption}

\begin{remark}[Functions on the vertices]
    Let $\cV$ be the set of extreme points of $\cX$. Our positive results (\Cref{theorem:main-eah-prec,theorem:main-prec}) only evaluate $\phi$ at extreme points, so they would operate identically if we instead defined our maps $\phi$ to be $\cV \to \cX$.
\end{remark}

The definition above places some minimal assumptions on the feature mapping $m$ to ensure that $\Phi^m$ is geometrically well behaved. Indeed, we first show that the set of transformations $\Phi^m$ under~\Cref{assumption:kernel-precise} is well-bounded; the proof is provided in~\Cref{sec:wellbounded}.

\begin{restatable}{lemma}{wellbounded}
    \label{lemma:wellbounded}
    Let $\cX \subseteq \R^d$ be a convex and compact set such that $\cB_r(\vec{0}) \subseteq \cX \subseteq \cB_R(\vec{0})$, with $R \geq 1$ and $r < R$. Suppose further that $\|m(\vx)\| \leq M$ for all $\vx \in \cB_R(\vec{0})$, with $M = M(R) \geq 1$; $\co m(\cB_r(\vec{0})) \supseteq \cB_\delta(\vec{0})$ for some $\delta = \delta(r) > 0$; and $m(\vec{0}) = \vec{0}$. Then,
    \begin{equation}
        \cB_{r'}(\vec{0}) \subseteq \Phi^m \subseteq \cB_{ R'}(\vec{0}),
    \end{equation}
    where $r' \defeq \nicefrac{r}{2M(R)}$ and $R' \defeq R \left( \frac{2\sqrt{d}}{\delta(r)} + 1 \right)$.
\end{restatable}

We are now ready to provide a canonical, concrete example of deviations that satisfy~\Cref{assumption:kernel} under~\Cref{assumption:kernel-precise}. As we alluded to earlier in our introduction, it is the family of low-degree polynomials; in particular, it will be convenient to work with the Legendre basis.

\begin{definition}
    \label{def:polys}
    Let $P_0(x) = 1$ and $P_1(x) = x$. The $(\ell+1)$th \emph{Legendre polynomial} is given by the recurrence $(\ell + 1) P_{\ell + 1}(x) - (2\ell+1) x P_\ell(x) + \ell P_{\ell - 1}(x) = 0$.
\end{definition}

These polynomials have a convenient orthogonality property over $[-1, 1]$:
\begin{equation}
    \label{eq:orthog}
    \int_{-1}^1 P_{\ell}(x) P_{\ell'}(x) dx = 
    \begin{cases}
        \frac{1}{2\ell + 1} & \text{if } \ell = \ell',\\
        0 & \text{otherwise}.
    \end{cases}
\end{equation}
For convenience, we shall consider the rescaled polynomial $\lineP_\ell \defeq \sqrt{2\ell + 1} P_{\ell}$, so that $\int_{-1}^1 \lineP_\ell(x)^2 dx = 1$. We now define
\begin{equation}
    \label{eq:polynomials}
    m(\vx) \defeq \left( \prod_{j=1}^d \lineP_{\ell_j}(\vx[j]) \right)_{ 1 \leq \ell_1 + \dots \ell_d \leq \ell}.
\end{equation}

We establish (in~\Cref{sec:aux}) that~\Cref{assumption:kernel-precise} encompasses the above mapping.

\begin{restatable}{proposition}{lowdeg}
    \label{lemma:lowdeg}
    Let $m : \cX \to \R^{k'}$ per~\eqref{eq:polynomials}, where $k' = \binom{d + \ell}{\ell} - 1$. $\linem : \vx \mapsto m( \sqrt{d} \vx)$ satisfies~\Cref{assumption:kernel-precise} with $M \leq d^{O(\ell)}$  and $\delta = \nicefrac{1}{M}$.
\end{restatable}
\section{Polynomial-time expected fixed points and semi-separation}

We will now start connecting the framework we laid out in~\Cref{sec:eah} with the problem of computing $\Phi$-equilibria. That a $\Phi$-equilibrium can be cast as~\eqref{eq:eah}---by linking $\cY$ to the set of deviations $\Phi$---is not hard to see, and will be spelled out in the next section. This section concerns the question of implementing $\epsilon$-$\either$ (per~\Cref{def:either}), which is the main precondition of~\Cref{theorem:eah}.

The key to implementing the $\epsilon$-$\either$ oracle, and the main subject of this section, is the notion of an \emph{$\epsilon$-expected fixed point (EFP)} (\Cref{def:EFP}). This relaxes the usual notion of a fixed point that was employed by~\citet{Zhang24:Efficient}, who observed that, for minimizing $\Phi$-regret, one can replace a fixed-point oracle---as in the canonical framework of~\citet{Gordon08:No}---by an expected relaxation (per~\Cref{def:EFP}) thereof. This is crucial because, unlike fixed points which are intractable beyond linear maps, there is a simple, $O(1/\epsilon)$-time algorithm for computing $\epsilon$-expected fixed points. When it comes to computing $\Phi$-equilibria in games, our contribution here is twofold.

\begin{enumerate}
    \item we give a $\poly(d, \log(1/\epsilon))$-time algorithm for computing an $\epsilon$-expected fixed point, and\label{item:fast}
    \item we show that expected fixed points can be naturally coupled with the $\eah$ framework, and in particular, with the recent generalization of~\citet{Daskalakis24:Efficient}.\label{item:furious}
\end{enumerate}
This section establishes~\Cref{item:fast}, while the next section formalizes~\Cref{item:furious}. Going back to~\Cref{sec:eah} and the $\either$ oracle, the connection with (expected) fixed points lies in the observation that, when it comes to $\Phi$-equilibria in games, the $\ger$ part of the oracle can be implemented by computing an expected fixed point. This will become clear in the upcoming section.

\begin{definition}[Expected fixed points]
    \label{def:EFP}
    Let $\cX \subseteq \R^d$ be convex and compact and a function $\phi : \cX \to \cX$ to which we are given oracle access. The $\eps$-{\em expected fixed point (EFP)} problem asks for a distribution $\mu \in \Delta(\cX)$ such that\footnote{The choice of the $1$-norm (instead of, say, another $p$-norm) here is unimportant, because one can always take $\eps$ to be exponentially small.} 
    \begin{equation}
        \label{eq:EFP}
    \norm{\E_{\vx \sim \mu}[\phi(\vx)-\vx]}_1 \le \eps.
    \end{equation}
\end{definition}

This definition departs from the usual notion of a fixed point by measuring the fixed-point error \emph{in expectation} over samples $\vx$ from $\mu$. \Cref{def:EFP} is natural in its own right, but our key motivation is computational: as we shall see, an expected fixed point can be computed in polynomial time (\Cref{th:efp}).

We first observe that any function $\phi : \cX \to \cX$ admits an \emph{exact} expected fixed point; crucially, unlike Brouwer's fixed-point theorem, we do not assume that $\phi$ is continuous, and so expected fixed points exist even when fixed points do not.

\begin{proposition}
    Every function $\phi : \cX \to \cX$ admits an exact solution to the EFP problem.
\end{proposition}

\begin{proof}
By induction on the dimension of $\cX$. For a $0$-dimensional set (a single point), the claim is trivial.
    The existence of an exact solution is equivalent to the statement that $\vec 0 \in \cZ := \co \{ f(\vx) - \vx : \vx \in \cX\}$. Suppose that no such solution exists. Then by the separating hyperplane theorem, there is a vector $\vv \in \R^d \setminus \vec 0$ with $\ip{f(\vx) - \vx, \vv} \ge 0$ for every $\vx \in \cX$. Now consider the face $\cX'$ of $\cX$ consisting of optimizers of $\ip{\vv, \vx}$, that is,
    \begin{align}
        \cX' := \{ \vx \in \cX : \ip{\vv, \vx} = \max_{\vx'\in\cX} \ip{\vv, \vx'}\}.
    \end{align}
    Then $\cX'$ has strictly smaller dimension than $\cX$ (because $\vv \ne \vec 0$), and for any $\vx \in \cX'$ we have that the inequality $\ip{f(\vx) - \vx, \vv} \ge 0$ is tight, because $\ip{\vx, \vv}$ is already maximized by such $\vx$, and therefore $f(\vx) \in \cX'$ as well. Thus, $f$ maps $\cX'$ to itself. Thus, by inductive hypothesis, there is an EFP supported on $\cX'$.
\end{proof}

We now turn to our main computational result regarding EFPs; namely, a polynomial-time algorithm based on $\eah$---in particular, its incarnation in~\Cref{alg:eah}.

\begin{theorem}
    \label{th:efp}
    Given oracle access to $\cX$ and any $\phi : \cX \to \cX$, there exists a $\poly(d, \log(1/\eps))$-time algorithm that computes an $\epsilon$-EFP of $\phi$.
\end{theorem}

\begin{proof}
    We observe that an EFP can be equivalently expressed through the optimization problem
    \begin{equation}
    \qq{find} \mu \in \Delta(\cX) \qq{s.t.} \E_{\vx \sim \mu} \ip{\vy, \phi(\vx) - \vx } \ge 0 \quad \forall \vy \in [-1, 1]^d.
\end{equation}
We will now apply~\Cref{theorem:eah}. The set $\cY \defeq [-1, 1]^d$ clearly admits a separation oracle ($\sep$). Further, for any $\vy \in [-1, 1]^d$, taking $\vx^* = \argmin_{\vx \in \cX} \langle \vy, \vx \rangle $ (using an optimization oracle for $\cX$) guarantees $\langle \vy, \phi(\vx^*) - \vx^* \rangle \geq 0 $ since $\phi(\vx^*) \in \cX$, thereby implementing the $\ger$ oracle. We thus find that the preconditions of~\Cref{theorem:eah} are satisfied, and $\eah$ (\Cref{alg:eah}) returns $\mu \in \Delta(\cX)$, with $\supp(\mu) \leq \poly(d, \log(1/\epsilon))$, such that
\begin{equation}
    \E_{\vx \sim \mu} \langle \vy, \phi(\vx) - \vx \rangle \geq - \epsilon \quad \forall \vy \in [-1, 1]^d.
\end{equation}
Taking $\vy = \sign( \E_{\vx \sim \mu} (\vx - \phi(\vx)))$ (coordinate-wise) completes the proof.
\end{proof}

As it will become clear, \Cref{th:efp} yields a polynomial-time implementation of the $\ger$ oracle in the context of~\Cref{sec:eah}, which can be employed in~$\eah$. With a slight modification in the proof of~\Cref{th:efp}, we shall see how one can also recover an $\epsilon$-$\either$ oracle (\Cref{def:either}), which will then enable us to harness~\Cref{theorem:either} for computing $\Phi$-equilibria in games. Following the nomenclature of~\citet{Daskalakis24:Efficient}, we refer to this oracle as a \emph{semi-separation oracle}.

\begin{definition}[Semi-separation oracle]
    \label{def:semiseparation}
    The \emph{semi-separation} problem is the following. Given a convex and compact $\cX$ and a function $\phi : \cX \to \R^d$, compute
    \begin{enumerate}
        \item {\em either} a distribution $\mu \in \Delta(\cX)$ such that $\|\E_{\vx \sim \mu}[\phi(\vx)-\vx] \|_1 \leq \epsilon$, \label{item:efp}
        \item {\em or} a point $\vx\in\cX$ with $\phi(\vx) \notin \cX$.\label{item:mem}
    \end{enumerate}
\end{definition}

Unlike~\Cref{def:EFP}, here we allow $\phi$ to map outside of $\cX$. This more general framing is essential to arrive at the $\either$ oracle. In particular, we note that~\Cref{item:mem} yields a hyperplane separating $\phi$ from the set of endomorphisms on $\cX$. Namely, since $\phi(\vx) \notin \cX$, we can use the separation oracle on $\cX$ to separate $\cX$ from $\phi(\vx)$; that is, there is a $\vec{w}$ such that $\langle \phi(\vx), \vec{w} \rangle > \langle \vx, \vec{w} \rangle$ for all $\vx \in \cX$. But this also implies that $\langle \phi(\vx), \vec{w} \rangle > \langle \phi'(\vx), \vec{w} \rangle$ for any endomorphism $\phi'$, as promised.

\begin{theorem}
    \label{theorem:semiseparation}
    Given oracle access to $\cX$ and $\phi$, there is a $\poly(d, \log(1/\epsilon))$-time algorithm for implementing the semi-separation oracle of~\Cref{def:semiseparation}.
\end{theorem}

\begin{proof}
As in the proof of~\Cref{th:efp}, we proceed by running the ellipsoid algorithm (per~\Cref{alg:eah}) on the problem
\begin{align}
    \qq{find} \vy\in [-1,1]^d \qq{s.t.} \ip{\vy, \phi(\vx) - \vx} \le -\eps \quad \forall \vx \in \cX. \label{eq:efp-constraint}
\end{align}
For any $\vy \in [-1, 1]^d$ during the execution of the ellipsoid, take $\vx^*(\vy) \in \argmin_{\vx \in \cX} \ip{\vy, \vx}$. If $\phi(\vx^*(\vy)) \notin \cX$, the algorithm can terminate and return $\vx^*(\vy)$. Otherwise, it follows that $\langle \vy, \phi(\vx^*(\vy)) - \vx^*(\vy) \rangle \geq 0$, by definition of $\vx^*$, and so we can use $\vx^*$ to get a separation oracle for~\eqref{eq:efp-constraint}.

Now, if every $\vx^*(\vy) \in \cX$ generated above satisfies the constraint $\phi(\vx^*(\vy)) \in \cX$, then~\Cref{alg:eah} returns a certificate of infeasibility for~\eqref{eq:efp-constraint} in $\poly(d, \log(1/\epsilon))$ time, which is an $\epsilon$-expected fixed point of $\phi$. On the other hand, if at some point there is $\vy \in [-1, 1]^d$ such that $\phi(\vx^*(\vy)) \notin \cX$, then the algorithm returns a point $\vx^*(\vy) \in \cX$ such that $\phi(\vx^*(\vy)) \notin \cX$. This completes the proof.
\end{proof}

This semi-separation oracle amounts to the $\epsilon$-$\either$ oracle needed in~\Cref{theorem:either}, as we shall see next in the context of games. Compared to the semi-separation oracle of~\citet{Daskalakis24:Efficient} that only works for linear functions, ours (\Cref{theorem:semiseparation}) places no restrictions on $\phi$. 

\section{A polynomial-time algorithm for $\Phi^m$-equilibria in games}

Armed with the powerful semi-separation oracle of~\Cref{theorem:semiseparation}, we now establish a polynomial-time algorithm for computing $\Phi^m$-equilibria in general multilinear games (\Cref{theorem:main-eah-prec}).

Let us recall the basic setting of an $n$-player multilinear game $\Gamma$. Each player $i \in [n]$ has a convex and compact strategy set $\cX_i \subseteq \R^{d_i}$ in isotropic position (\Cref{sec:prel}). Player $i$ has a utility function $u_i : \cX_1 \times \dots \times \cX_n \to \R$ that is linear in $\cX_i$, so that $u_i(\vx) = \langle \vec{g}_i, \vx_i \rangle$ for some $\vec{g}_i = \vec{g}_i(\vx_{-i}) \in \R^{d_i}$. 
Furthermore, for each player $i \in [n]$, we let $\Phi^{m_i} \subseteq \cX_i^{\cX_i}$ be the $k_i$-dimensional set of deviations in the sense of~\Cref{assumption:kernel}; that is, there exists a function $m_i \in \cX_i \to \R^{k_i'}$, with $k_i = k_i' \cdot d_i + d_i$, such that for each $\phi_i \in \Phi^{m_i}$ and $\vx_i \in \cX_i$, the function output $\phi_i(\vx_i)$ can be expressed as the matrix-vector product $\mK_i(\phi_i) m_i(\vx_i) + \cons_i$ for some matrix $\mK_i \in \R^{d_i \times k_i'}$ and $\cons_i \in \R^{d_i}$. It is assumed throughout that $\Phi^{m_i}$ contains the identity map. For notational simplicity, we let $k \defeq \sum_{i=1}^n k_i$ and $d \defeq \sum_{i=1}^n d_i$.

In this context, we next state the main result of this section, and proceed with its proof.

\begin{theorem}[Precise version of~\Cref{theorem:main-eah}]
    \label{theorem:main-eah-prec}
    Consider an $n$-player multilinear game $\Gamma$ such that, for each player $i \in [n]$, we are given $\poly(n, k)$-time algorithms for the following:
\begin{itemize}
\item an oracle to compute the gradient, that is, the vector $\vec{g}_i = \vec{g}_i(\vx_{-i}) \in \R^{d_i}$ for which $\ip{\vg_i(\vx_{-i}), \vx_i} = u_i(\vx)$ for all $\vx \in \cX_1 \times \dots \times \cX_n$ (polynomial expectation property); and
\item a membership oracle for the strategy set $\cX_i$, assumed to be in isotropic position.
\end{itemize}
Suppose further that each $k_i$-dimensional set $\Phi^{m_i}$ satisfies~\Cref{assumption:kernel-precise} and $\|\vec{g}_i \| \leq B$. Then, an $\eps$-approximate $\Phi^m$-equilibrium of $\Gamma$ can be computed in $\poly\qty(n, k, \log(B/\eps))$ time.
\end{theorem}

\begin{proof}
An $\eps$-approximate $\Phi^m$-equilibrium of $\Gamma$  is a distribution $\mu \in \Delta(\cX_1 \times \dots \times \cX_n)$ such that
\begin{align}
\E_{\vx \sim \mu} \qty[u_i(\phi_i(\vx_i), \vx_{-i}) - u_i(\vx)] \le \eps\label{eq:phi-eqm}
\end{align}
for every player $i \in [n]$ and deviation $\phi_i \in \Phi^{m_i}$. Using multilinearity and~\Cref{assumption:kernel}, it suffices to find a distribution $\mu\in\Delta(\cX_1 \times \dots \times \cX_n)$ satisfying
\begin{align}
\E_{\vx \sim \mu} \left[ \sum_{i =1}^n \ip{\vg_i(\vx_{-i}), \mK_i m_i(\vx_i) + \cons_i - \vx_i} \right] \le \eps \label{eq:phi-sp}
\end{align}
for every $(\mK_1(\phi_1), \dots, \mK_n(\phi_n))$ and $(\cons_1(\phi_1), \dots, \cons_n(\phi_n))$, where $(\phi_1, \dots, \phi_n) \in \Phi^{m}$. (This derivation uses the fact that $\Phi^{m_i}$ contains the identity map.) We will now apply~\Cref{theorem:either} with respect to $\R^d \supseteq \cX \defeq \cX_1 \times \dots \times \cX_n$ and 
$$\R^{k} \supseteq \cY \defeq \{ (\mK_1, \cons_1, \dots, \mK_n, \cons_n) : \mK_i m_i(\vx_i) + \cons_i \in \cX_i \quad \forall \vx_i \in \cX_i \}. $$
By the polynomial expectation property, we can evaluate the term $\sum_{i =1}^n \ip{\vg_i(\vx_{-i}), \mK_i m_i(\vx_i) + \cons_i - \vx_i}$, for each $\vx \in \cX$, in $\poly(n, k)$ time. It thus suffices to show how to implement the $\epsilon$-$\either$ oracle, which yields a separation oracle for the program
\begin{alignat}{9}
\qq{find} \mK_1, \cons_1, \dots, \mK_n, \cons_n \qq{s.t.} \\
\sum_{i =1}^n \ip{\vg_i(\vx_{-i}), \mK_i \vm_i(\vx_i) + \cons_i - \vx_i} \ge - \eps &\quad\forall \vx \in \cX_1 \times \dots \times \cX_n, \label{eq:phi-sp}
\\
\mK_i \vm_i(\vx_i) + \cons_i \in \cX_i &\quad \forall\vx_i \in \cX_i.
\end{alignat}

Consider any $\R^{k} \ni \phi = (\mK_1, \cons_1, \dots, \mK_n, \cons_n)$. We apply the semi-separation oracle of~\Cref{theorem:semiseparation} for each function $\vx_i \mapsto \mK_i m_i(\vx_i) + \cons_i$. This returns \emph{either} an $\epsilon'$-expected fixed point, that is, a distribution $\nu_i \in \Delta(\cX_i)$ such that $$\left\| \E_{\vx_i \sim \nu_i} [ \mK_i \vm_i(\vx_i) + \cons_i - \vx_i ] \right\|_1 \leq \epsilon',$$
\emph{or} a point $\vx_i \in \cX_i$ such that $\mK_i m_i(\vx_i) + \cons_i \notin \cX_i$. If any of those semi-separation oracles returned $\vx_i \in \cX_i$ with $\mK_i m_i(\vx_i) + \cons_i \notin \cX_i$, we can use it to obtain a hyperplane separating $(\mK, \cons)$ from the set of deviations $\cY$. Otherwise, let $\nu \defeq \nu_1 \times \dots \times \nu_n \in \Delta(\cX_1) \times \dots \times \Delta(\cX_n)$ be the induced product distribution. Then, we have
\begin{align}
    & \E_{\vx \sim \nu}\sum_{i = 1}^n \ip{\vg_i(\vx_{-i}), \mK_i \vm_i(\vx_i) + \cons_i - \vx_i} 
    = \sum_{i =1}^n  \big\langle\E_{\vx \sim \nu} \vg_i(\vx_{-i}), \E_{\vx_i \sim \nu_i}[\mK_i \vm_i(\vx_i) + \cons_i - \vx_i]\big\rangle \leq n B \eps', \label{eq:eqm-final}
\end{align}
where we used the fact that $\nu$ is a product distribution in the equality above. Thus, we have identified an $(\epsilon' n B)$-approximate good-enough-response, yielding an $\epsilon$-$\either$ oracle by rescaling $\epsilon'$, and the proof follows from~\Cref{theorem:either}.
\end{proof}

\begin{algorithm}[!ht]
\caption{Polynomial-time algorithm for $\Phi^m$-equilibria}
\label{alg:main-eah}
\SetKwInOut{Input}{Input}
\SetKwInOut{Output}{Output}
\SetKw{Input}{Input:}
\SetKw{Output}{Output:}
\Input{
    \begin{itemize}[noitemsep,topsep=0pt]
        \item An $n$-player multilinear game $\Gamma$
        \item A precision parameter $\epsilon > 0$
        \item A membership oracle for each $\cX_i$
        \item An oracle for computing the gradient $\vec{g}_i = \vec{g}_i(\vx_{-i}) \in \R^{d_i}$ for each $i \in [n]$
        \item A $k_i$-dimensional set $\Phi^{m_i}$ under~\Cref{assumption:kernel-precise} for each $i \in [n]$
    \end{itemize}
}
\Output{An $\epsilon$-approximate $\Phi^m$-equilibrium of $\Gamma$ in $\poly(k, \log(1/\epsilon))$ time}\\
Define $G : \R^{d} \to \R^k$ such that $\langle G(\vx), (\mK, \cons) \rangle = \sum_{i =1}^n \ip{\vg_i(\vx_{-i}), \mK_i m_i(\vx_i) + \cons_i - \vx_i}$\;
Use the semi-separation oracle of~\Cref{theorem:semiseparation} to construct an $\epsilon$-$\either$ oracle $\cO$\;
 Apply~\Cref{alg:gen-eah} with $\cO$ as the $\epsilon$-$\either$ oracle
\end{algorithm}

It is worth stressing that it is crucial for our proof that the \emph{expected} VI problem (\emph{cf}.~\citet{Zhang25:Expected}) above corresponds to a game. It allows each player to be treated {\em independently}, which yields a {\em product distribution} $\nu = \nu_1 \times \dots \times \nu_n$ when we apply the semi-separation oracle of~\Cref{theorem:semiseparation} (for each player). That $\nu$ is a product distribution is crucial to implement the separation oracle for the dual because it allows us to push the expectation into the inner product in \eqref{eq:eqm-final}, as we saw in the last step of the proof.

\section{An efficient online algorithm for minimizing $\Phi^m$-regret}
\label{sec:reg}

We now switch gears to the online learning setting, recalled in~\Cref{sec:gordon}. Our main result, \Cref{theorem:main-prec}, is an efficient online algorithm for minimizing $\Phi^m$-regret with respect to any $\poly(d)$-dimensional set $\Phi^m$ (under~\Cref{assumption:kernel-precise}), which applies even in the adversarial regime.

In what follows, we build on the framework of~\citet{Daskalakis24:Efficient}, itself a refinement of the template of~\citet{Gordon08:No}. As we have seen, \citet{Daskalakis24:Efficient} showed that separating even over the set of linear endomorphisms is hard. In light of this, they proceed as follows. Instead of operating over the set of linear endomorphisms, their key idea is to consider a sequence of ``shell sets,'' each of which contains the original set. Each shell set must also satisfy two basic properties:
\begin{itemize}
    \item it is sufficiently structured so that it is possible to optimize over that set, and
    \item it contains a transformation with a fixed point inside $\cX$.
\end{itemize}
Here, we show that by replacing fixed points with \emph{expected} fixed points in the above template, it is possible to extend their main result to handle any $\poly(d)$-dimensional set under~\Cref{assumption:kernel-precise}.

\paragraph{Overview} Our main construction is~\Cref{alg:main}. It is an instantiation of~$\shellgd$ (\Cref{sec:shellgd}), which is projected gradient descent but with the twist that the constraint set is changing over time---reflecting the fact that a new shell set is computed at every round. To execute $\shellgd$, $\shellproj$ (\Cref{sec:shellproj}) provides an efficient projection oracle together with an approximate expected fixed point thereof, which is ultimately the output of our $\Phi^m$-regret minimizer. $\shellproj$ crucially relies on $\shellelips$, introduced next in~\Cref{sec:shellellips}. It strengthens our semi-separation oracle of~\Cref{theorem:semiseparation} by again using expected fixed points. \Cref{sec:put} combines those ingredients to arrive at our main result (\Cref{theorem:main-prec}).

\subsection{Shell ellipsoid}
\label{sec:shellellips}

Continuing from our semi-separation oracle of~\Cref{theorem:semiseparation}, $\shellelips$ (\Cref{alg:shellellipsoid}) takes a step further: it takes as input a convex set of transformations $\cF \subseteq \cB_D(\vec{0})$---for which we have efficient oracle access, unlike $\enfuns$---and returns \emph{either} a function $\phi \in \cF$ and an $\epsilon$-expected fixed point thereof in $\Delta(\cX)$, \emph{or} it provides a certificate---in the form of a polytope expressed as the intersection of a polynomial number of halfspaces---establishing that $\vol(\cF \cap \enfuns) \approx 0$. $\shellelips$ will be used later as part of the $\shellproj$ algorithm so as to shrink the shell set.

\begin{lemma}
    \label{lemma:shellellipsoid}
    For any $k$-dimensional convex set $\cF \subseteq \cB_{D}(\vec{0})$ with efficient oracle access and $\epsilon > 0$, $\shellelips(\cF)$ (\Cref{alg:shellellipsoid}) runs in time $\poly(k, \log(1/\epsilon), \log D)$, and 
    \begin{itemize}
        \item either it returns a transformation $\phi \in \cF$ with an $\epsilon$-expected fixed point in $\cX$,
        \item or it returns a polytope $\cQ \subseteq \R^k$, specified as the intersection of at most $\poly( k, \log(1/\epsilon), \log D)$ halfspaces, such that $\Phi^m \subseteq \cQ$ and $\vol(\cQ \cap \cF) < \epsilon$.
    \end{itemize}
\end{lemma}

Coupled with~\Cref{theorem:semiseparation} pertaining to the semi-separation oracle, the proof of correctness of~\Cref{lemma:shellellipsoid} is immediate. That $\cQ$ can be expressed as the intersection of a polynomial number of halfspaces follows from the usual analysis of ellipsoid, as in~\citet[Lemma 4.2]{Daskalakis24:Efficient}.

\begin{algorithm}[!ht]
\caption{$\shellelips(\cF)$}
\label{alg:shellellipsoid}
\SetKwInOut{Input}{Input}
\SetKwInOut{Output}{Output}
\SetKw{Input}{Input:}
\SetKw{Output}{Output:}
\Input{
    \begin{itemize}[noitemsep,topsep=0pt]
        \item Oracle access to convex set $\cX \subseteq \R^d$
        \item Oracle access to a $k$-dimensional convex set $\cF \subseteq \cB_D(\vec{0})$
        \item Precision parameter $\epsilon > 0$
    \end{itemize}
}
Initialize $\cE \defeq \cB_D(\vec{0})$ and $\cQ \defeq \R^k$\;
 \While{$\vol(\cE) \geq \epsilon$} {
    Set $\phi \in \cQ \cap \cF$ as the center of $\cE$\;
    Run the semi-separation oracle of~\Cref{theorem:semiseparation} with respect to $\phi$\;
    \If{it returned an $\epsilon$-expected fixed point $\mu \in \Delta(\cX)$ of $\phi$}{
        \textbf{return} $\phi$\;
    }
    \Else{
        Let $H$ be the halfspace returned by~\Cref{theorem:semiseparation} that separates $\phi$ from $\enfuns$\;
        Set $\cQ \defeq \cQ \cap H$\;
    }
    Set $\cE$ to be the minimum volume ellipsoid containing $\cQ \cap \cF$
 }
 \textbf{return} $\cQ$
\end{algorithm}

\subsection{Shell gradient descent}
\label{sec:shellgd}

Instead of minimizing external regret with respect to the set $\enfuns$, which is hard even under linear endomorphisms~\citep{Daskalakis24:Efficient}, the overarching idea is to run (projected) gradient descent but with respect to a sequence of changing shell sets, $(\tilY^{(t)})_{t=1}^T$, of $\enfuns$ (each of which contains $\enfuns$); this process, called~$\shellgd$, is given in~\Cref{alg:shellgd}. So long as $\enfuns \subseteq \tilY^{(t)}$, $\shellgd$ indeed minimizes external regret with respect to deviations in~$\enfuns$---of course, $\shellgd$ is not a genuine regret minimizer for $\enfuns$ in that it is allowed to output strategies not in $\enfuns$, but~\Cref{lemma:shellgd} below is in fact enough for the purpose of minimizing $\Phi^m$-regret.

\begin{lemma}[\citep{Daskalakis24:Efficient}]
    \label{lemma:shellgd}
    Suppose that the sequence of shell sets $(\tilY^{(t)})_{t=1}^T$ is such that $\enfuns \subseteq \tilY^{(t)} \subseteq \cB_D(\vec{0})$ for all $t \in [T]$. For any sequence of utilities $\vec{U}^{(1)}, \dots, \vec{U}^{(T)} \in [-1, 1]^k$, $\shellgd$ (\Cref{alg:shellgd}) satisfies
    \begin{equation}
        \max_{\vy^* \in \enfuns} \sum_{t=1}^T \langle \vy^* - \vy^{(t)}, \vec{U}^{(t)} \rangle \leq \frac{D^2}{2\eta} + \eta \sum_{t=1}^T \| \vec{U}^{(t)} \|^2.
    \end{equation}
\end{lemma}

\begin{algorithm}[!ht]
\caption{$\shellgd$~\citep{Daskalakis24:Efficient}}
\label{alg:shellgd}
\SetKwInOut{Input}{Input}
\SetKwInOut{Output}{Output}
\SetKw{Input}{Input:}
\SetKw{Output}{Output:}
\Input{Learning rate $\eta$, convex and compact sets $\tilY^{(1)}, \dots, \tilY^{(T)} \subseteq \cB_D(\vec{0})$}\;
Initialize $\vy^{(0)} \in \tilY^{(1)}$ and $\vec{U}^{(0)} \defeq \vec{0}$\;
 \For{$ t=1, \dots, T$} {
    Obtain efficient oracle access to $\tilY^{(t)}$\;
    Update $\vy^{(t)} \defeq \Pi_{\tilY^{(t)}}( \vy^{(t-1)} + \eta \vec{U}^{(t-1)})$\;
    Output $\vy^{(t)}$ as the next strategy and receive feedback $\vec{U}^{(t)} \in [-1, 1]^k$
 }
\end{algorithm}

\subsection{Shell projection}
\label{sec:shellproj}

To implement $\shellgd$, we will make use of $\shellproj$, the algorithm that is the subject of this subsection. There are two main desiderata for the sequence of shell sets taken as input in $\shellgd$. First, each shell set must be structured or simple enough to allow projecting onto it---this is the whole rationale of expanding $\enfuns$ through shell sets. But, of course, this is not enough, for one could just consider the entire space. The second crucial concern is that each transformation given by~$\shellgd$ needs to admit (approximate) expected fixed points, so as to use the framework of~\citet{Gordon08:No} (\Cref{theorem:gordon}) and minimize $\Phi^m$-regret. \Cref{lemma:shellproj} below, concerning $\shellproj$, shows how to accomplish that goal; its proof is similar to that of~\citet[Theorem 4.4]{Daskalakis24:Efficient}.

\begin{lemma}
    \label{lemma:shellproj}
    Let $\cX$ be a convex and compact set such that $\cB_{r}(\vec{0}) \subseteq \cX \subseteq \cB_R(\vec{0})$ and $\cM$ be a convex set such that $\enfuns \subseteq \cM \subseteq \cB_D(\vec{0})$. For any $\phi \in \cB_D(\vec{0}) \subseteq \R^k$ and $\epsilon > 0$, $\shellproj$ (\Cref{alg:shellproj}) runs in time $\poly(k, 1/\epsilon, R/r, D)$ and returns
    \begin{enumerate}
        \item a shell set $\tilPhi$ satisfying $\enfuns \subseteq \tilPhi$, expressed by intersecting $\cM$ with at most $\poly(d, k, 1/\epsilon, R/r, D)$ halfspaces, and\label{item:invar}
        \item a transformation $\tilphi \in \tilPhi$ such that $\| \tilphi - \Pi_{\tilPhi}(\phi) \| \leq \epsilon$, together with an $\epsilon$-expected fixed point of $\tilphi$, $\mu \in \Delta(\cX)$.\label{item:proj}
    \end{enumerate}
\end{lemma}

\begin{algorithm}[!ht]
\caption{$\shellproj_\Phi(\phi)$ projects $\phi$ to a shell of $\Phi$}
\label{alg:shellproj}
\SetKwInOut{Input}{Input}
\SetKwInOut{Output}{Output}
\SetKw{Input}{Input:}
\SetKw{Output}{Output:}
\Input{
    \begin{itemize}[noitemsep,topsep=0pt]
        \item Convex and compact set $\cX \subseteq \R^d$ such that $\cB_r(\vec{0}) \subseteq \cX \subseteq \cB_R(\vec{0})$
        \item Convex set $\cM$ such that $\enfuns \subseteq \cM \subseteq \cB_D(\vec{0})$
        \item Transformation $\phi \in \cB_D(\vec{0})$
        \item Precision parameter $\epsilon > 0$
    \end{itemize}
}
\Output{
    \begin{itemize}[noitemsep,topsep=0pt]
        \item Convex set $\tilPhi$ such that $\enfuns \subseteq \tilPhi \subseteq \cM$
        \item Transformation $\tilphi \in \tilPhi$ such that $\| \tilphi - \Pi_{\tilPhi}(\phi) \| \leq \epsilon$
        \item An $\epsilon$-expected fixed point $\mu \in \Delta(\cX)$ of $\tilphi$
    \end{itemize}
}
Set $\epsilon' = \frac{\epsilon r}{32 M(R) D^2}$\;
Initialize $\tilPhi \defeq \cM$\;
 \For{$q = 0, \dots$ incremented by $\delta \defeq \nicefrac{\epsilon}{4D} $} {
    Run $\shellelips( \tilPhi \cap \cB_q(\phi))$ with precision $\vol(\cB_{\epsilon'}(\cdot))$\;
    \If{it finds $\tilphi$ with an $\epsilon$-expected fixed point $\mu \in \Delta(\cX)$}{
        \textbf{return} $\tilPhi, \tilphi, \mu$
    }
    \Else{
        Let $\cQ$ be the polytope returned by $\shellelips$\;
        Set $\tilPhi \defeq \tilPhi \cap \cQ$
    }
 }
\end{algorithm}

\subsection{Putting everything together}
\label{sec:put}

We now combine all the previous pieces to obtain an efficient algorithm for minimizing $\Phi^m$-regret---when $\Phi^m$ is $\poly(d)$-dimensional---under a general convex and compact set $\cX$. The overall construction is depicted in~\Cref{alg:main}. In effect, it runs $\shellgd$ with respect to the sequence of shell sets $(\tilPhi^{(t)} )_{t=1}^T$. Indeed, by the correctness guarantee of $\shellproj$ (\Cref{item:invar} of~\Cref{lemma:shellproj}), we have the invariance $\Phi(\cX) \subseteq \tilPhi^{(t)}$ for all $t \in [T]$. Furthermore, \Cref{item:proj} of~\Cref{lemma:shellproj} implies that $(\mK^{(t+1)}, \cons^{(t+1)}) \in \tilPhi^{(t+1)}$, returned by $\shellproj$ in~\Cref{alg:main}, is within distance $\epsilon$ of the projection prescribed by~$\shellgd$. As a result, we can apply~\Cref{lemma:shellgd} (up to some some slackness proportional to $\epsilon$) to bound the external regret $\reg^{(T)}_{\Phi^m}$ of $((\mK^{(t)}, \cons^{(t)}))_{t=1}^T$ with respect to comparators from $\enfuns$; combined with the fact that $\mu^{(t)} \in \Delta(\cX)$ is an $\epsilon$-expected fixed point of the function $\vx \mapsto \mK^{(t)} m(\vx) + \cons^{(t)}$ (as promised by~\Cref{item:proj}), it follows that the $\Phi^m$-regret of the learner (\Cref{alg:main}) can be bounded by $\reg^{(T)}_{\Phi^m} + \epsilon T$ (as in~\Cref{theorem:gordon}). We thus arrive at our main result.

\begin{theorem}[Precise version of~\Cref{theorem:main1}]
    \label{theorem:main-prec}
    Let $\cX \subseteq \R^d$ be a convex and compact set in isotropic position for which we have a membership oracle. \Cref{alg:main} has per-round running time of $\poly(k, T)$ and guarantees average $\Phi^m$-regret of at most $\poly(k) / \sqrt{T}$, where $k$ is the dimension of $\Phi^m$ under~\Cref{assumption:kernel-precise}.
\end{theorem}

Unlike the algorithm of~\citet{Daskalakis24:Efficient}, a salient aspect of~\Cref{alg:main} is that it outputs a sequence of \emph{mixed} strategies in $\Delta(\cX)$. As we saw earlier in~\Cref{sec:gordon}, this turns out to be necessary: \citet{Zhang24:Efficient} showed that a learner restricted to output strategies in $\cX$ cannot efficiently minimize $\Phi$-regret even with respect to low-degree swap deviations (assuming $\PPAD \neq \P$).

\begin{algorithm}[!ht]
\caption{$\Phi^m$-regret minimizer for convex strategy sets}
\label{alg:main}
\SetKwInOut{Input}{Input}
\SetKwInOut{Output}{Output}
\SetKw{Input}{Input:}
\SetKw{Output}{Output:}
\Input{
    \begin{itemize}[noitemsep,topsep=0pt]
        \item Convex and compact set $\cX \subseteq \R^d$ in isotropic position
        \item $k$-dimensional set $\Phi^m$ under~\Cref{assumption:kernel-precise} with respect to $m : \cX \to \R^{k'}$, where $k = k' \cdot d + d$
        \item time horizon $T \in \N$
    \end{itemize}
}
\Output{An efficient $\Phi^m$-regret minimizer for $\cX$}\\
Set the learning rate $\eta \propto \frac{1}{\sqrt{T}}$ and $\epsilon = \nicefrac{1}{\poly(k,T)}$ to be sufficiently small\;
Initialize $\mu^{(1)} \in \Delta(\cX)$ and $\mK^{(1)} \defeq \mI_{d \times k'}$ to be the identity map and $\cons^{(1)} \defeq \vec{0}$\;
Initialize $\cM \defeq \cB_{R}(\vec{0})$ for a large enough $R \leq \poly(k)$\;
 \For{$ t=1, \dots, T$} {
    Output $\mu^{(t)} \in \Delta(\cX)$ and receive feedback $\vu^{(t)} \in [-1, 1]^d$\;
    Define $\R^{d \times k' + d} \ni \mU^{(t)} \defeq (\E_{\vx^{(t)} \sim \mu^{(t)}} \vu^{(t)} \otimes m(\vx^{(t)}), \vec{u}^{(t)})$\;
    Set $\tilPhi^{(t+1)}, (\mK^{(t+1)}, \cons^{(t+1)}), \mu^{(t+1)} \defeq \shellproj_{\Phi}((\mK^{(t)}, \cons^{(t)}) + \eta \mU^{(t)})$ with input $\cM$ and precision $\epsilon$, where $\mu^{(t+1)} \in \Delta(\cX)$ is an $\epsilon$-expected fixed point of $\vx \mapsto \mK^{(t+1)} m(\vx) + \cons^{(t+1)}$ 
 }
\end{algorithm}

Finally, we conclude by providing a lower bound that matches our upper bound (\Cref{theorem:main-prec}) up to a constant factor in the exponent of $k$. It is based on the following normal-form lower bound due to~\citet{Dagan24:From} and~\citet{Peng24:Fast}.

\begin{theorem}[\citep{Dagan24:From,Peng24:Fast}]
    \label{theorem:lowerknown}
    Consider a learner operating on the simplex $\Delta(\cA)$. For any $T < |\cA|/4$, there is an adversary that forces the swap regret of the learner to be $\Omega(\log^{-6}T)$.
\end{theorem}

We observe that there is a simple way to parameterize the above lower bound in terms of the dimension of the set of deviations:

\begin{corollary}
    \label{cor:parlower}
    Consider a learner operating on the simplex $\Delta(\cA)$. There is a $k$-dimensional set of deviations $\Phi \subseteq \Delta(\cA)^{\Delta(\cA)}$ such that for any $T < \sqrt{k}/4$, there is an adversary that forces the $\Phi$-regret of the learner to be $\Omega(\log^{-6} T)$.
\end{corollary}

Indeed, one can first identify an arbitrary subset $\cA'$ of $\cA$ with cardinality $\sqrt{k}$, and then employ the adversary of~\Cref{theorem:lowerknown} with respect to $\cA'$ while rendering all other actions dominated by assigning to them very small utility. That $\Phi$ is $k$-dimensional in this case follows because the set of stochastic matrices mapping $\Delta(\cA')$ to $\Delta(\cA')$---which contains all relevant swap deviations---is $(\sqrt{k})^2$-dimensional.

Combining~\Cref{cor:parlower} with the recent reduction of~\citet{Daskalakis24:Lower}, which embeds the normal-form game lower bound of~\Cref{theorem:lowerknown} into an extensive-form game, we arrive at~\Cref{theorem:mainlower}, which we restate below.

\lowerbound*
\section{Conclusions and open problems}

In summary, we established efficient algorithms for minimizing $\Phi$-regret and computing $\Phi$-equilibria with respect to any set of deviations with a polynomial dimension. For the online learning setting, our upper bounds are tight up to constant factors in the exponents, crystallizing for the first time a family of deviations that characterizes the learnability of $\Phi$-regret. 

There are many important avenues for future research. First, we did not attempt to optimize the (polynomial) dependence of the running time (in~\Cref{theorem:main-prec,theorem:main-eah-prec}) on $k$ and $d$; improving the overall complexity of our algorithms is an interesting direction. Moreover, developing more practical algorithms---that refrain from using ellipsoid---would also be a valuable contribution. In particular, are there polynomial-time algorithms for computing $\Phi$-equilibria without resorting to the $\eah$ framework? But the most pressing open question is to understand the complexity of computing (normal-form) correlated equilibria in the centralized model.

\section*{Acknowledgments}
T.S. is supported by the Vannevar Bush Faculty Fellowship ONR N00014-23-1-2876, National Science Foundation grants RI-2312342 and RI-1901403, ARO award W911NF2210266, and NIH award A240108S001. B.H.Z. is supported 
by the CMU Computer Science Department Hans Berliner
PhD Student Fellowship. E.T, R.E.B., and V.C. thank the Cooperative AI Foundation, Polaris Ventures (formerly the Center for
Emerging Risk Research) and Jaan Tallinn’s donor-advised fund at Founders Pledge for financial
support. E.T. and R.E.B. are also supported in part by the Cooperative AI PhD Fellowship. G.F is supported by the National Science Foundation grant CCF-2443068. We are indebted to Constantinos Daskalakis and Noah Golowich for many insightful discussions concerning the complexity of computing expected fixed points.

\bibliography{dairefs}

\clearpage
\appendix

\section{Sufficiency of regret minimization in isotropic position}
\label{sec:isotropic}

Throughout the paper, we have assumed that we are minimizing $\Phi$-regret with respect to a convex set $\cX$ that is in isotropic position. \Cref{lemma:isotropic} below shows that this is without any loss. The argument here is similar to~\citet[Lemma A.1]{Daskalakis24:Efficient}, with the minor modification that we need to account for mixed strategies.

\begin{lemma}
    \label{lemma:isotropic}
    Let $\cX \subseteq \R^d$ be a convex and compact set such that $\cX \subseteq \cB_R(\vec{0})$. Let $\psi : \cX \to \tilde{\cX}$ be an invertible affine transformation such that $\tilde{\cX} = \psi(\cX)$ is in isotropic position. Suppose that we have a regret minimizer $\regbox_{\tilde{\cX}}$ for $\tilde{\cX}$ that incurs $\phireg_{\tilde{\cX}}^{(T)}$. Then, using $\poly(d)$ time in each round, we can construct a regret minimizer $\regbox_{\cX}$ for $\cX$ that incurs $\phireg^{(T)}_{\cX} \leq 2R \sqrt{d} \cdot \phireg^{(T)}_{\tilde{\cX}}$.
\end{lemma}

\begin{proof}
    Let $\psi(\vx) \defeq \mA \vx + \vec{b}$ for an invertible $\mA \in \R^{d \times d}$ and $\vec{b} \in \R^d$. Let $\vu^{(1)}, \dots, \vu^{(T)}$ be the sequence of utilities given as input to $\regbox_{\cX}$. We then provide as input to $\regbox_{\tilde{\cX}}$ the sequence of utilities
    \begin{equation}
        \label{eq:new-utils}
    \tilde{\vu}^{(t)} \defeq \frac{1}{2R \sqrt{d}} ( \mA^{-1} )^\top \vu^{(t)} \quad t = 1, \dots, T.
    \end{equation}
    Since $\tilde{\cX}$ contains the ball $\cB_1(\vec{0})$ and $\cX \subseteq \cB_R(\vec{0})$, it follows that $\| \mA^{-1} \vx \| \leq 2R$ for any $\vx \in \cB_1(\vec{0})$, which implies that the spectral norm of $\mA^{-1}$ is at most $2 R$. As a result, assuming that $\vu^{(t)} \in [-1, 1]^d$ for all $t \in [T]$, it follows that the utilities constructed in~\eqref{eq:new-utils} are also in $[-1, 1]^d$.
    
    Now, suppose that $\regbox_{\tilde{\cX}}$ returns the sequence of strategies $\tilde{\mu}^{(1)}, \dots, \tilde{\mu}^{(T)} \in \Delta(\tilde{\cX})$. We define, for each $t \in [T]$, $\mu^{(t)} \defeq \psi^{-1}(\tilde{\mu}^{(t)})$ as the next strategy. Consider any $\phi \in \Phi$, and define $\tilde{\phi} : \tilde{\cX} \ni \vx \mapsto \psi(\phi(\psi^{-1}(\vx))) \in \Phi(\tilde{\cX})$. Then,
    \begin{align}
        &\sum_{t=1}^T \left\langle \vu^{(t)}, \E_{\tilde{\vx}^{(t)} \sim \tilde{\mu}^{(t)}} \psi^{-1}(\tilde{\vx}^{(t)}) - \phi(\psi^{-1}(\tilde{\vx}^{(t)})) \right\rangle \\
        &= 2R \sqrt{d} \sum_{t=1}^T \left\langle \frac{1}{2R \sqrt{d}} 
(\mA^{-1})^\top \vu^{(t)}, \E_{\tilde{\vx}^{(t)} \sim \tilde{\mu}^{(t)}} \mA \psi^{-1}(\tilde{\vx}^{(t)}) - \mA \phi(\psi^{-1}(\tilde{\vx}^{(t)})) \right\rangle \\
&= 2R \sqrt{d} \sum_{t=1}^T \left\langle \frac{1}{2R \sqrt{d}} 
(\mA^{-1})^\top \vu^{(t)}, \E_{\tilde{\vx}^{(t)} \sim \tilde{\mu}^{(t)}} (\mA \psi^{-1}(\tilde{\vx}^{(t)}) + \vec{b}) - (\mA \phi(\psi^{-1}(\tilde{\vx}^{(t)})) + \vec{b}) \right\rangle \\
&= 2R \sqrt{d} \sum_{t=1}^T \left\langle \tilde{\vu}^{(t)}, \E_{\tilde{\vx}^{(t)} \sim \tilde{\mu}^{(t)}} \tilde{\vx}^{(t)}  - \tilde{\phi}(\tilde{\vx}^{(t)}) \right\rangle \leq 2R \sqrt{d} \cdot \phireg^{(T)}_{\tilde{\cX}}.
    \end{align}
\end{proof}
\section{Geometric properties of $\Phi^m$}
\label{sec:wellbounded}

In this section, we establish that the set $\Phi^m$ per~\Cref{assumption:kernel-precise} is geometrically well behaved, which is necessary to execute the ellipsoid algorithm (as well as the online learning setting). In particular, our goal is to prove \Cref{lemma:wellbounded}.

Below, for convex and compact $\cA, \cB \subseteq \R^d$, we use the notation 
\begin{equation}
    \Phi^m(\cA, \cB) \defeq \left\{ (\mK, \cons) \in \R^{k + d} : \mK m(\vx) + \cons \in \cB \quad \forall \vx \in \cA \right\}.
\end{equation}

\begin{lemma}
    \label{lemma:trivial}
    Let $\cA, \cB, \cC, \cD$ be convex and compact sets. If $\cA \supseteq \cC$ and $\cB \subseteq \cD$, then $\Phi^m(\cA, \cB) \subseteq \Phi^m(\cC, \cD)$.
\end{lemma}

\begin{proof}
    Consider any $(\mK, \cons) \in \Phi^m(\cA, \cB)$. By definition, it holds that $\mK \vx + \cons \in \cB$ for all $\vx \in \cA$. Since $\cC \subseteq \cA$, it follows that $\mK \vx + \cons \in \cB$ for all $\vx \in \cC$, and in particular, $\mK \vx + \cons \in \cD$ since $\cB \subseteq \cD$.
\end{proof}

\begin{lemma}
    \label{lemma:boundM}
    Let $\cX \subseteq \R^d$ be a convex and compact set such that $\cB_r(\vec{0}) \subseteq \cX \subseteq \cB_R(\vec{0})$, with $R \geq 1$. Suppose further that $\|m(\vx)\| \leq M$ for all $\vx \in \cB_R(\vec{0})$, where $M = M(R) \geq 1$. Then, $\Phi^m \supseteq \cB_{r'}(\phi_{\vec{0}})$, where $r' \defeq \nicefrac{r}{2M(R)}$ and $\phi_{\vec{0}} \defeq (\mathbf{0}, \vec{0})$ is the constant transformation $\vx \mapsto \vec{0}$.
\end{lemma}

\begin{proof}
    By~\Cref{lemma:trivial}, it suffices to prove $\cB_{r'}(\phi_{\vec{0}}) \subseteq \Phi^m(\cX, \cB_r(\vec{0}))$. Consider any $(\mK, \cons) \in \cB_{r'}(\phi_{\vec{0}})$, which means that $\| \mK \|^2_F + \|\cons \|^2 \leq ( \nicefrac{r}{2M(R)})^2$. Then, for any $\vx \in \cX$,
    \begin{equation}
        \| \mK m(\vx) + \cons \| \leq \| \mK \|_F \|m(\vx)\| + \|\cons  \| \leq r.
    \end{equation}
    This means that $\cB_{r'}(\phi_{\vec{0}}) \subseteq \Phi^m(\cX, \cB_r(\vec{0}))$, and the proof follows.
\end{proof}

\begin{lemma}
    \label{lemma:fulldim}
    Suppose that $\co m(\cB_r(\vec{0})) \supseteq \cB_\delta(\vec{0})$ for some $\delta = \delta(r) > 0$ and $m(\vec{0}) = \vec{0}$. Then, assuming that $r < R$,
    \begin{equation}
        \Phi^m( \cB_r(\vec{0}), \cB_R(\vec{0})) \subseteq \cB_{ R'}(\vec{0}),
    \end{equation}
    where $R' \defeq R \left( \frac{2\sqrt{d}}{\delta(r)} + 1 \right)$.
\end{lemma}

\begin{proof}
    Consider any $(\mK, \cons) \in \Phi^m( \cB_r(\vec{0}), \cB_R(\vec{0}))$. By definition, we have $\| \mK m(\vx) + \cons \| \leq R$ for all $\vx \in \cB_r(\vec{0})$. Since $m(\vec{0}) = \vec{0}$, it follows that $\|\cons \| \leq R$. Thus, $\| \mK m(\vx) \| \leq \|\mK m(\vx) + \cons \| + \|\cons \| \leq 2 R$ for all $\vx \in \cB_r(\vec{0})$. Now, let $\vec{x}' \in \R^{k'}$ with $\|\vec{x}' \| = 1$ be such that $\| \mK \vec{x}' \| = \| \mK\|$, where $\| \mK \|$ is the spectral norm of $\mK$. Since we have assumed that $\co m(\cB_r(\vec{0})) \supseteq \cB_\delta(\vec{0})$, it follows that there exist $\lambda_1, \dots, \lambda_{k'+1}$, with $\lambda_1, \dots, \lambda_{k' + 1} \geq 0$ and $\sum_{j=1}^{k'+1} \lambda_j = 1$, and $\vx_1, \dots, \vx_{k'+1} \in \cB_r(\vec{0})$ (by Carath\'eodory's theorem) such that $\sum_{j=1}^{k'+1} \lambda_j m(\vx_j) = \delta \vx'$. As a result,
    \begin{equation}
        \delta \|\mK\| = \| \mK (\delta \vx') \| = \left\| \mK \left( \sum_{j=1}^{k'+1} \lambda_j m(\vx_j) \right) \right\| \leq \sum_{j=1}^{k' + 1} \lambda_j \|\mK m(\vx_j) \| \leq 2R.
    \end{equation}
    Finally, we have $\|\mK\|_F \leq \sqrt{d} \|\mK\|$, and the claim follows.
\end{proof}

\begin{proof}[Proof of~\Cref{lemma:wellbounded}]
    The claim follows directly by combining~\Cref{lemma:trivial,lemma:boundM,lemma:fulldim}.
\end{proof}

\section{Further omitted proofs}
\label{sec:aux}

Finally, this section provides the proof of~\Cref{lemma:lowdeg}, which we restate below. For completeness, we have also included the usual version of $\eah$ (\Cref{alg:eah}, subsumed by~\Cref{alg:gen-eah}), which we used earlier in~\Cref{theorem:semiseparation,th:efp}.

\lowdeg*

For the proof, we will use a simple, auxiliary lemma.

\begin{lemma}
    \label{lemma:range}
    Let $X$ be a random variable such that $\E[X] = 0$, $\V[X] = 1$, and $X \in [-R, R]$ almost surely. Then, $\pr[X \geq \nicefrac{1}{R}] > 0$.
\end{lemma}

\begin{proof}[Proof of~\Cref{lemma:lowdeg}]
    It is clear that $\linem(\vec{0}) = \vec{0}$. The bound on $M$ is also immediate. We thus focus on proving that $\delta \defeq 1/M$ suffices.
    
    For the sake of contradiction, suppose that $\co \linem(\cB_1(\vec{0}))$ does not contain $\vx'$ for some $\vx' \in \R^{k'}$ with $\| \vx' \| \leq \delta$. Then, we consider a hyperplane that separates $\co \linem(\cB_1(\vec{0}))$ from $\vx'$, and we let $\vec{v}$ be the normal vector to that hyperplane, so that $\langle \vec{v}, \vx' \rangle > \langle \vec{v}, \linem(\vx) \rangle$ for all $\vx \in \cB_1(\vec{0})$. Now, let $\cU$ be the uniform product distribution over $[-1, 1]^d$. By~\eqref{eq:orthog}, we have $\E_{\vx \sim \cU} [m(\vx)] = 0$ and $\E_{ \vx \sim \cU} [ m(\vx) m(\vx)^\top] = \mI_{k' \times k'}$ (by ortogonality). As a result, we have $\E_{ \vx \sim \cU} [  \langle \vec{v}, m(\vx) \rangle^2 ] = \mathbb{V}_{\vx \sim \cU}[ \langle \vec{v}, m(\vx) \rangle ] = \|\vec{v} \|^2 = 1$. \Cref{lemma:range}, applied for the random variable $\langle \vec{v}, m(\vx) \rangle$ with range $[-M, M]$, implies that there exists $\vx \in [-1, 1]^d$ such that $\langle \vec{v}, m(\vx) \rangle \geq 1/M$, which in turn implies that there exists $\linex \in \cB_1(\vec{0})$---namely, $\linex \defeq \vx/\sqrt{d}$---such that $\langle \vec{v}, \linem(\linex) \rangle \geq 1/M$. But this yields $\delta \leq \langle \vec{v}, \linem(\linex) \rangle < \langle \vec{v}, \vx' \rangle \leq \|\vec{v} \| \|\vx'\| = \delta$, a contradiction.
\end{proof}

\begin{algorithm}[!ht]
\caption{Ellipsoid against hope ($\eah$)~\citep{Papadimitriou08:Computing}}
\label{alg:eah}
\SetKwInOut{Input}{Input}
\SetKwInOut{Output}{Output}
\SetKw{Input}{Input:}
\SetKw{Output}{Output:}
\Input{\begin{itemize}[noitemsep,topsep=0pt]
        \item Parameters $R_y, r_y > 0$ such that $\cB_{r_y}(\cdot) \subseteq \cY \subseteq \cB_{R_y}(\vec{0})$
        \item Precision parameter $\epsilon > 0$
        \item Parameter $B > 0 $ such that $\| G(\vx) \| \leq B$ for all $\vx \in \cX$
        \item $\ger$ oracle for~\eqref{eq:eah}
        \item $\sep$ oracle for $\cY$
    \end{itemize}
}
\Output{A sparse, $\epsilon$-approximate solution $\mu \in \Delta(\cX)$ of~\eqref{eq:eah}}
Initialize the ellipsoid $\cE \defeq \cB_{R_y}(\vec{0})$\;
\While{$\vol(\cE) \geq \vol(\cB_{\epsilon/B}(\cdot))$}{
    Let $\vy \in \R^k$ be the center of $\cE$ \;
    \eIf{$\vy \in \cY$}{
        Let $\vx \in \cX$ be a good-enough-response with respect to $\vy$ (via the $\ger$ oracle)\;
        Update $\cE$ to the minimum volume ellipsoid containing $\cE \cap \{ \vy \in \R^k : \langle \vy, G(\vx) \rangle \leq 0 \}$\;
    }{
        Let $H$ be the halfspace that separates $\vy$ from $\cY$ (via the $\sep$ oracle) \;
        Update $\cE$ to the minimum volume ellipsoid containing $\cE \cap H$\;
    }
}
Let $\vx^{(1)}, \dots, \vx^{(T)}$ be the $\ger$ oracle responses produced in the process above\;
Define $\mG \defeq [G(\vx^{(1)}) \mid \hdots \mid G(\vx^{(T)})] \in \R^{k \times T} $\;
Compute a solution $\vec{\lambda}$ to the convex program
    $$\qq{find} \vec{\lambda} \in \Delta^T \qq{s.t.} \min_{\vy \in \cY} \vec{\lambda}^\top \mG^\top \vy \geq - \epsilon$$
\Return{$\Delta(\cX) \ni \mu \defeq \sum_{t=1}^T \lambda^{(t)} \mu(\vx^{(t)})$}
\end{algorithm}

\end{document}